\documentclass[11pt]{article}
\usepackage[normalem]{ulem}

\usepackage[pagenumbers]{cvpr} 

\definecolor{cvprblue}{rgb}{0.21,0.49,0.74}
\usepackage[pagebackref,breaklinks,colorlinks,allcolors=cvprblue]{hyperref}
\usepackage{amsmath,amsthm,amssymb,amsfonts, fancyhdr, xcolor, comment, graphicx, environ,mathtools,bm,bbm}
\usepackage{algorithm}
\usepackage{algpseudocode}
\usepackage{ stmaryrd }
\usepackage{caption}
\usepackage{subcaption}

\usepackage{csquotes}
\usepackage{multirow}
\usepackage{booktabs}
\usepackage{dsfont}

\newcommand{\R}{\mathbb{R}}
\newcommand{\G}{\mathcal{G}}

\usepackage[subtle]{savetrees}

\newtheorem{definition}{Definition}
\newtheorem{theorem}{Theorem}[section]
\newtheorem{corollary}{Corollary}[theorem]

\newtheorem{proposition}[theorem]{Proposition}
\newtheorem{remark}{Remark}
\usepackage{color}   
\hypersetup{
    colorlinks=true, 
    linktoc=all,    
    linkcolor=blue,
}

\usepackage{soul}

\newcommand\blfootnote[1]{%
  \begingroup
  \renewcommand\thefootnote{}\footnote{#1}%
  \addtocounter{footnote}{-1}%
  \endgroup
}

\def\paperID{*****} 
\def\confName{CVPR}
\def\confYear{2026}

\title{QuadSync: Quadrifocal Tensor Synchronization via Tucker Decomposition}

\author{ Daniel Miao \thanks{School of Mathematics, University of Minnesota (\href{mailto:miao0022@umn.edu}{miao0022@umn.edu}, \href{mailto:lerman@umn.edu}{lerman@umn.edu}) }
\hspace{1cm}
Gilad Lerman\footnotemark[1]
\hspace{1cm}
Joe Kileel\thanks{Department of Mathematics and Oden Institute for Computational Engineering and Sciences, University of Texas at Austin (\href{mailto:jkileel@math.utexas.edu}{jkileel@math.utexas.edu})}}

\begin{document}
\maketitle
\bibliographystyle{unsrt}

\begin{abstract}
In structure from motion, 
quadrifocal tensors capture more information than their pairwise counterparts (essential matrices), yet they have often been thought of as impractical and only of 
theoretical interest. In this work, we challenge such beliefs by providing a new framework to recover $n$ cameras from the corresponding collection of quadrifocal tensors.  We form the block quadrifocal tensor and show that it admits a Tucker decomposition whose factor matrices are the stacked camera matrices, and which thus has a multilinear rank of (4,~4,~4,~4) independent of $n$. We develop the first synchronization algorithm for quadrifocal tensors, using Tucker decomposition, alternating direction method of multipliers, and iteratively reweighted least squares.  We further establish relationships between the block quadrifocal, trifocal, and bifocal tensors, and introduce an algorithm that jointly synchronizes these three entities. Numerical experiments demonstrate the effectiveness of our methods on modern datasets, indicating the potential and importance of using higher-order information in synchronization.
\end{abstract}
\blfootnote{Complementary code can be found at \url{https://github.com/dmiao153/QuadSync}.}
\section{Introduction}\label{sec:intro}
Structure from motion (SfM) has been one of the most active research areas in 3D computer vision. SfM concerns the reconstruction of a 3D model of a scene from a set of 2D images taken from different views. Typically the pipeline of SfM includes feature detection and matching, relative pose estimation, synchronization, and reconstruction. While many works have explored synchronization based on pairwise measurements, few have investigated the possibility of incorporating higher order measurements into the SfM pipeline, even though higher order measurements may have the capability to improve reconstruction quality through stronger constraints and replicated information. One obstacle has been that higher-order measurements are less understood compared to their counterparts; there has also not been sufficient dedicated work on computing them. The authors of \cite{thirthala2012radial} state that, ``\textit{Currently the quadrifocal and mixed trifocal tensors are useful only from a theoretical stand-point}''. In this work, we lay the theoretical groundwork for synchronizing quadrifocal tensors by characterizing properties of a collection of quadrifocal tensors, while also providing two promising practical algorithms for advancing current SfM systems using quadrifocal tensors. 

\subsection{Relevant Previous Work}

In SfM, the classical synchronization methods are incremental methods, such as Bundler \cite{snavely2006photo} and COLMAP \cite{schonberger2016structure}.  Cameras are registered sequentially while the 3D scene is reconstructed using bundle adjustment. 
The order in which cameras are registered may greatly impact the reconstruction quality due to error accumulation.  Also, bundle adjustment becomes computationally expensive with many cameras. 

 Global synchronization methods have been proposed whereby cameras are processed simultaneously.  Global methods can be split into different categories. 
 Some directly process a collection of fundamental or essential matrices. This is usually achieved by enforcing algebraic constraints like low rankness \cite{sengupta2017new, kasten2019gpsfm, kasten2019algebraic, geifman2020averaging} or properties of the fundamental matrix \cite{madhavan2025recovery}. Others extract the relative rotations and relative locations from  essential matrices. One can use the structure of $SO(3)$ to obtain a global estimate of the camera orientations \cite{govindu2004lie, chatterjee2013efficient, hartley2013rotation,chatterjee2017robust,shi2020message,arie2012global}. Then, one can retrieve the global location configuration using the rotation estimates \cite{wilson2014robust, ozyesil2015stable, goldstein2016shapefit, zhuang2018baseline, li2025cycle}. Alternatively, one may synchronize over $SE(3)$ and retrieve rotations and translations simultaneously \cite{arrigoni2016spectral, rosen2019se, cucuringu2012sensor, briales2017cartan}. 
 Modern global-based  pipelines include GLOMAP \cite{pan2024global} and Theia \cite{sweeney2015theia}. 
 
 Recently, there have been efforts toward investigating the use of higher-order information in synchronization, like trifocal tensors and higher-order cycles \cite{li2024efficient, miao2024tensor, duncan2025higher}. These works point to possible benefits of using higher-order information in SfM. 

Separately, there have been attempts to utilize quadrifocal relationships in multiview geometry. \cite{hruby2023four} studies the minimal problem for estimating a special quadrifocal tensor wherein one of the views is a radial camera. \cite{comport2007accurate, comport2010real} exploit quadrifocal relationships to develop a trajectory estimation problem, yet do not directly operate on quadrifocal tensors. To the best of our knowledge, no prior works have focused on the global synchronization of a collection of quadrifocal tensors.

\subsection{Our Contributions}
Our contributions can be summarized as follows.
\begin{enumerate}
    \item We develop new and strong theory: a system of algebraic constraints for a set of quadrifocal tensors.  It is expressed as a low rank condition on a block tensor that we introduce.
    \item We develop the first global synchronization algorithm for quadrifocal tensors, and a joint synchronization scheme that combines this with pairwise and triplewise measurements. 
    \item We demonstrate the promise of our algorithms through numerical experiments and show the effectiveness and practicality of using higher-order measurements in synchronization.
\end{enumerate}

\section{Background}\label{sec:background}

Given a collection of $n$ images $I_1,...,I_n$ of a 3D scene, standard pinhole cameras associated with each image are modeled as $3\times 4$ matrices $P_i = K_i R_i [I_{3\times 3} \mid -t_i]$, where $K_i \in \R^{3\times 3}$ is the calibration matrix, $R_i \in SO(3)$ is the orientation, and $t_i \in \R^3$ is the location of the camera in the global coordinate system. 
Let $\mathbf{X} \in \R^4$ be a world point in homogeneous coordinates, and $\mathbf{x_i} = P_i \mathbf{X}$ its projection into the image plane of $I_i$ in homogeneous coordinates. The fundamental matrix $F_{ij}$ relates points $\mathbf{x_i}$ and $\mathbf{x_j}$ that correspond to the same world point in the image planes $I_i$ and $I_j$ respectively via
\begin{equation*}
    \mathbf{x_i}^TF_{ij}\mathbf{x_j} = 0.
\end{equation*} 
Analogously, the trifocal tensor $T_{ijk} = \{T_1,T_2,T_3\} \in \R^{3\times 3 \times 3}$ relates corresponding points across three images $I_i,I_j,I_k$, such that 
\begin{equation*}
    [\mathbf{x_j}]_\times \left( \sum_{w=1}^3 (\mathbf{x_i})_w T_w \right) [\mathbf{x_k}]_\times = 0_{3\times 3},
\end{equation*}  
where $[ \cdot ]_{\times}$ denotes the corresponding skew-symmetric matrix. 

We also recall how global synchronization for fundamental matrices and trifocal tensors can be conducted. For synchronizing a collection of fundamental matrices, \cite{sengupta2017new} introduces the \textit{multiview matrix of fundamentals} $\mathcal{F}^n$ (or the \textit{$n$-view fundamental matrix} or the \textit{block fundamental matrix}). It is formed by concatenating the $\binom{n}{2}$ fundamental matrices associated with $n$ cameras into a $3n \times 3n$ matrix $\mathcal{F}^n$, thought of as an $n \times n$ block matrix where the $ij^{th}$ block is $F_{ij} \in \mathbb{R}^{3n \times 3n}$. With a suitable set of nonzero scales on $F_{ij}$, then $\mathcal{F}^n$ admits a factorization $\mathcal{F}^n = A+A^T$ where rank$(A) = 3$ for not all collinear cameras, so that $\operatorname{rank}(\mathcal{F}^n) = 6$. 
The constraints on $\mathcal{F}^n$ were further explored in \cite{kasten2019algebraic}, \cite{kasten2019gpsfm}, \cite{geifman2020averaging}. 
The full characterization of the $n$-view fundamental matrix from \cite{kasten2019gpsfm} is summarized by the following theorem. 
\begin{theorem}[\cite{kasten2019gpsfm}]
    A given matrix $\mathcal{F}^n\in \mathbb{R}^{3n \times 3n}$ is consistent as an $n$-view fundamental matrix with some set of $n$ cameras whose centers are not all collinear if and only if:
    \begin{enumerate}
        \item[1.] $rank(\mathcal{F}^n) = 6$ and $\mathcal{F}^n$ has exactly $3$ positive and $3$ negative eigenvalues. 
        \item[2.] $rank(\mathcal{F}^n_{i,:}) = 3$ for all $i=1,...,n$, where $\mathcal{F}^n_{i,:}$ denotes the $i^{th}$ row block of $\mathcal{F}^n$.
    \end{enumerate}
\end{theorem}

Recently \cite{miao2024tensor} explores the interesting higher-order setting with trifocal tensors. The \textit{block trifocal tensor} $\mathcal{T}^n \in \R^{3n\times 3n\times 3n}$ was introduced and formed by stacking the $3\times 3 \times 3$ trifocal tensors $T_{ijk}$ in the $ijk^{th}$ position in $\mathcal{T}^n$.  It was shown that $\mathcal{T}^n$ admits an exact low rank Tucker decomposition when its blocks are suitably scaled. The characterization of $\mathcal{T}^n$ is as follows.

\begin{theorem}[\cite{miao2024tensor}]\label{thrm:tft}
    The block trifocal tensor admits a Tucker factorization, $\mathcal{T}^n = \G_T \times_1 \mathcal{P}\times_2 \mathcal{C} \times_3 \mathcal{C}$, where $\G_T \in \R^{6\times 4 \times 4}, \mathcal{P} \in \mathbb{R}^{3n \times 6}$, and $\mathcal{C} \in \mathbb{R}^{3n \times 4}$. If the $n$ cameras that produce $\mathcal{T}^n$ are not all collinear, then mlrank$(\mathcal{T}^n) = (6,4,4)$. Otherwise, mlrank$(\mathcal{T}^n) \preceq (6,4,4)$. 
\end{theorem}
\subsection{Quadrifocal Tensors}
Analogous to  fundamental matrices and  trifocal tensors, quadrifocal tensors encode information across four views. In tensor notation, for views $I_i,I_j,I_k,I_l$, corresponding image points are related through the quadrifocal tensor $Q_{ijkl} \in \mathbb{R}^{3 \times 3 \times 3 \times 3}$ via 
\begin{align*}
\mathbf{x_i}^a\mathbf{x_j}^b\mathbf{x_k}^c\mathbf{x_l}^d \epsilon_{apw}\epsilon_{bqx}\epsilon_{cry}\epsilon_{dsz}(Q_{ijkl})^{pqrs} = 0_{wxyz},
\end{align*}
where $\epsilon$ is the Levi-Civita tensor and there are implicit summations over repeated indices.
Quadrifocal tensors also relate corresponding image lines or mixtures of lines and points on lines. From the camera matrices $P_i,P_j,P_k,P_l$, the corresponding quadrifocal tensor $Q_{ijkl}$ can be calculated directly through 
\begin{equation}\label{eq:qfromp}
    (Q_{ijkl})^{pqrs} = \det 
    \begin{bmatrix}
        P_i^p \\
        P_j^q \\
        P_k^r \\
        P_l^s\\
    \end{bmatrix}.
\end{equation}
\noindent Like trifocal tensors and fundamental matrices, quadrifocal tensors are only well-defined up to nonzero scale since the camera matrices $P_i$ are themselves only well-defined up to nonzero scales.  We will need to find appropriate scales below.

There are certain interesting characteristics of quadrifocal tensors that motivate our choice to use them for synchronization. 
Firstly, quadrifocal tensors capture the interactions of four views, and thus encode  complex geometric information.  The quadrifocal tensor includes triple and pairwise information too. \cite{miao2024tensor} showed the effectiveness of using redundant information in trifocal tensors for better averaging of translations, and quadrifocal tensors stand to further improve results. 
Secondly, there may be more flexibility for estimating quadrifocal tensors, especially in homogeneous scenes. In line correspondence relationships of quadrifocal tensors, the lines in the image planes do not necessarily have to correspond to the same world line. The quadrifocal tensor can also be estimated from $6$ point correspondences. We refer to the following for detecting and matching lines across views \cite{liu20233d, pautrat2023gluestick}. 
Thirdly, in the case of quadrifocal tensors all views are treated equally, unlike in the trifocal tensor case where there is a view treated specially. This might lead to a more stable algorithm in synchronization. 
Lastly, but not least, estimating a quadrifocal tensor requires sufficiently many inlier features across four views.  This implicitly enforces consistency of the viewing graph through matched points and/or lines. Intentionally using quadrifocal tensors for synchronization may therefore help us to start synchronization from a cleaner graph, especially in the dense viewing graph case. 
We refer to \cite{hartley2003multiple, hartley1998computation, oeding2017quadrifocal, shashua2000structure} for more  on the properties and estimation of quadrifocal tensors. We also refer to \cite{elqursh2011line,kuang2013pose,kukelova2017clever,miraldo2018minimal,kileel2018distortion} for more on solving minimal problems related to computer vision. We refer to \cite{kileel2017minimal, kileel2018distortion, oeding2017quadrifocal, aholt2014ideal,snapshot2025} for results on minimal problems and the characterization of higher order tensors from an algebraic geometry perspective.

\subsection{Tucker Decomposition}
In this section, we review the Tucker decomposition and the multilinear rank of a tensor. Let $\mathcal{X} \in \R^{M_1 \times M_2 \times \cdots \times M_N}$ be an order $N$ tensor, i.e., an $N$ dimensional array. The mode-$i$ flattening (or matricization) $\mathcal{X}_{(i)} \in \R ^{M_i \times (\underset{j\not = i}{\prod}M_j)}$ is the rearrangement of $\mathcal{X}$ into a matrix by taking the slices along the $i$-th mode of the tensor to form rows of a matrix in lexicographic order. The Frobenius norm of a tensor is defined as $\|\mathcal{X}\| = (\sum_{i_1,...,i_N} \mathcal{X}_{i_1i_2...i_N}^2)^{1/2} = \|\mathcal{X}_{(i)}\|_F$. Let $R_i$ denote the matrix rank of $\mathcal{X}_{(i)}$.  Then the multilinear rank of $\mathcal{X}$ is mlrank$(\mathcal{X}) = (R_1,R_2,...,R_N)$. The $i$-th mode tensor-matrix product $\times_i$ is multiplication between tensor $\mathcal{X}$ and $U\in \R^{m\times M_i}$ is $\mathcal{X} \times_i U \in \mathbb{R}^{M_1 \times \dots \times M_{i-1} \times m \times M_{i+1} \times \dots \times M_N}$ given by
\begin{align*}
    (\mathcal{X} \times_i U)_{j_1\cdots j_{i-1}kj_{i+1}\cdots j_N} = \sum_{j_i=1}^{M_i}\mathcal{X}_{j_1j_2\cdots j_N}U_{kj_i}.
\end{align*}
A Tucker decomposition $\mathcal{X}$ is any tensor factorization of the following form 
\begin{align*}
    \mathcal{X} = \mathcal{G} \times_1 U_1 \times_2 U_2 \times_3 \dots \times_N U_N =: \llbracket \mathcal{G};~U_1,~U_2,~\dots~,~U_N \rrbracket,
\end{align*}
where $\mathcal{G} \in \R^{R_1\times R_2 \times \cdots R_N}$ is a core tensor, and $U_i \in \R^{I_i \times R_i}$ are called factor matrices. One standard way to obtain a Tucker decomposition for $\mathcal{X}$ is the higher-order singular value decomposition (HOSVD). In HOSVD, the factor matrices $U_i$ are the $R_i$ leading left singular vectors of $\mathcal{X}_{(i)}$. The core tensor can then be obtained as $\mathcal{G} = \mathcal{X} \times_1 U_1^T \times_2 \cdots \times_N U_N^T$. Similar to the matrix case, HOSVD can be used to project onto the set of lower multilinear rank tensors $A = \{\mathcal{X} \in \R^{M_1\times M_2 \times \cdots \times M_N} \mid \text{mlrank}(\mathcal{X}) = (r_1,r_2,...,r_N)\}$ for $r_i \leq R_i$. Then, HOSVD will satisfy a quasi-optimal property, where if $X^* = \arg \min_{X \in A} \|X - \mathcal{X}\|_F$ and $X'$ is the tensor obtained  from truncating HOSVD to the ranks $r_1,r_2,...,r_n$, then \begin{align*}
    \|\mathcal{X} - X'\|_F \leq \sqrt{N}\|\mathcal{X} - X^*\|_F.
\end{align*}
We refer to \cite{kolda2009tensor} for more information on tensor decomposition.
\section{The Block Quadrifocal Tensor}\label{sec:theory}

We first introduce the block quadrifocal tensor, which is a novel construction. Given $n$ cameras $\{P_i\}_{i=1}^n$, we form the \textit{block quadrifocal tensor} $\mathcal{Q}^n \in \R^{3n \times 3n \times 3n \times 3n}$ by stacking the quadrifocal tensors along the four modes.  Thus, the $ijkl^{th}$ block of $\mathcal{Q}^n$ is the quadrifocal tensor $Q_{ijkl} \in \mathbb{R}^{3 \times 3 \times 3 \times 3}$ corresponding to the cameras $P_i,~P_j,~P_k,~P_l$ calculated using equation \eqref{eq:qfromp}.

\subsection{Properties of the Block Quadrifocal Tensor}

Next, we develop theory for the block quadrifocal tensor $\mathcal{Q}^n$. We establish interesting properties of it, including low multilinear rank, low projection rank, a simple relationship to camera poses, and the sufficiency of the low multilinear rank to determine the scales of each block in $\mathcal{Q}^n$. All proofs are supplied in the supplementary materials.  We start with the following important characterization of the block quadrifocal tensor, in terms of Tucker decomposition and   multilinear rank.  
\begin{theorem}\label{thrm:tucker decomposition}
Let $\mathcal{Q}^n$ be a block quadrifocal tensor. 
Then there a choice of blockwise nonzero scales such that $\mathcal{Q}^n = \mathcal{G}_Q \times_1 C \times_2 C \times_3 C \times_4 C$, where $C \in \R^{3n \times 4}$ is the stacked camera matrix and $\mathcal{G}_Q \in \mathbb{R}^{4\times 4 \times 4 \times 4}$ is a constant sparse tensor with all entries in $\{-1,0,1\}$. When the cameras that produce $\mathcal{Q}^n$ do not all share the same camera center, then mlrank$(\mathcal{Q}^n) = (4,~4,~4,~4)$. 
\end{theorem}

\begin{remark}\label{remark: collinear}
    The last sentence of the theorem shows an advantage of $\mathcal{Q}^n$ over the $n$-view fundamental matrix $\mathcal{F}^n$ and the block trifocal tensor $\mathcal{T}^n$.  Both $\mathcal{F}^n$ and $\mathcal{T}^n$ experience a rank drop as soon as the corresponding cameras are collinear.  In that case, rank$(\mathcal{F}^n) = 4$ (see \cite{geifman2020averaging}) and mlrank$(\mathcal{T}^n) = (5,4,4)$ (shown in the supplementary materials).  Synchronizing based on fundamental matrices and trifocal tensors then requires additional procedures like constructing virtual cameras. 
\end{remark}
\noindent Given the explicit Tucker factorization of the block quadrifocal tensor $\mathcal{Q}^n$, there is a simple procedure to retrieve cameras from $\mathcal{Q}^n$ up to an invertible $4\times 4$ linear transformation, or a projective  transformation of world coordinates. 
\begin{corollary}[Projective Reconstruction] Given $\mathcal{Q}^n$ with appropriately chosen blockwise scales, one can retrieve the $n$ camera matrices up to a global projective frame ambiguity, through the higher-order singular value decomposition. One simply takes the 4 singular vectors corresponding to the largest singular values in any flattening. 
\end{corollary}

\begin{remark}
    Intuitively, Theorem~\ref{thrm:tucker decomposition} gives a stronger constraint compared to the two view and three view case, and thereby shows an advantage of using higher-order measurements for synchronization.  
    Roughly speaking, valid block quadrifocal tensors are restricted to a relatively smaller subset inside the full tensor space, as compared to their counterparts. This can be seen by counting the codimensions of the sets of low rank tensors, where $Q, T, E$ here will  represent the set of low Tucker rank tensors/matrices for quadrifocal, trifocal, and essential matrices, respectively. Here,  $\text{codim}_{(\R^{3n})^{\otimes 4}}(Q) > \text{codim}_{(\R^{3n})^{\otimes 3}}(T) > \text{codim}_{\R^{3n \times 3n}}(E)$. A degrees of freedom count shows  $\text{codim}_{(\R^{3n})^{\otimes 4}}(Q) = \Omega(n^4)$, $\text{codim}_{(\R^{3n})^{\otimes 3}}(T)= \Omega(n^3)$, and $\text{codim}_{\R^{3n \times 3n}}(E) = \Omega(n^2)$.  
\end{remark}
The paper \cite{aholt2014ideal} uses the projection rank (P-Rank) of a single trifocal tensor. We recall the definition of  the projection rank of a tensor and show that the block quadrifocal tensor has a low projection rank independent of $n$. 
We also establish the projection rank for the block trifocal tensor later in the paper. 
\begin{definition}
The \textup{projection rank} is defined as the tuple of ranks of the various projections from an order $n$ tensor to matrices by contraction with any generic tuple of $n-2$ vectors, where $n\geq 3$.
\end{definition}

For example, given $Q=\mathcal{Q}^n$, its projection rank is a tuple of 6 numbers, corresponding to the ranks of the matrices $\sum_{i,j} x_i y_j Q_{ij::}$, $\sum_{i,j} x_i y_j Q_{i:j:}$, $\sum_{i,j} x_i y_j Q_{i::j}$, $\sum_{i,j} x_i y_j Q_{:ij:}$, $\sum_{i,j} x_i y_j Q_{:i:j}$, $\sum_{i,j} x_i y_j Q_{::ij}$ for generic vectors $x_i, y_j \in \R^{3n}$. 

\begin{theorem}\label{thrm: prank of qft}
  Let $\mathcal{Q}^n$ be a block quadrifocal tensor with appropriately chosen blockwise scales. 
  Then P-Rank($\mathcal{Q}^n$)=$(2,~2,~2,~2,~2,~2)$. 
\end{theorem}

Note that this is not implied by the low multilinear rank; it depends on the structure of the core tensor in Theorem~\ref{thrm:tucker decomposition}.  By optimizing with the fixed core, and preserving the symmetries in the factor matrices, we will  enforce the P-Rank constraint. 

The block quadrifocal tensor contains explicit information of the two view and three view geometry through the next result. 
\begin{proposition}\label{prop:additional properties}
    The block quadrifocal tensor $Q = \mathcal{Q}^n$ admits the following properties.
    \begin{enumerate}
        \item The super-diagonal blocks vanish, $Q_{iiii} = 0$ for all $i=1,...,n$. 
        \item The blocks where three of the indices are the same correspond to the camera center of view $i$ in view $j$. 
        \item The blocks where two of the indices are the same correspond to elements in the trifocal tensor $T_{ijk}$. 
        \item The blocks where two of the indices are the same and the other two are also the same correspond to elements in the fundamental matrix $F_{ij}$. 
    \end{enumerate}
\end{proposition}

Lastly, we stress that quadrifocal tensors can only be estimated up to an unknown nonzero scale.  The nice theoretical properties rely on choosing a suitable set of scales.  Thus it is important to understand how our constraints determine the unknown scales. In the following theorem, the sufficiency of the low multilinear rank constraint to determine the scales on the block quadrifocal tensor is established, so that camera poses can be uniquely determined.   This theorem supports our algorithm development in the next section.  The notation $\odot_b$ denotes blockwise multiplication.
\begin{theorem}\label{thrm:quadscales}
    Let $\mathcal{Q}^n \in \R^{3n \times 3n \times 3n \times 3n}$ be a block quadrifocal tensor corresponding to $n\geq 5$  generic cameras. Let $\Lambda \in \R^{n\times n \times n \times n}$ be a block scaling with $\lambda_{ijkl} \not = 0$ if and only if $i,j,k,l$ are not all equal. If $\Lambda \odot_b \mathcal{Q}^n \in \R^{3n\times 3n \times 3n \times 3n}$ has multilinear rank $(4,4,4,4)$, then there exists $\alpha, \beta, \gamma, \delta \in \R^n$ such that $\lambda_{ijkl} = \alpha_i \beta_j \gamma_k \delta_l$ whenever $i,j,k,l$ are not all the same. That is, $\Lambda$ and $\alpha \otimes \beta \otimes \gamma \otimes \delta$ are equal away from the super diagonal.
\end{theorem}

\section{Method}\label{sec:algorithm}
The input to our synchronization method will be a collection of estimated quadrifocal, trifocal tensors, and fundamental matrices. There may be many missing measurements, and we assume all missing blocks are filled in with zeros. Let $\Omega = \{(i,j,k,l) \mid Q_{ijkl} \not = 0_{3\times 3 \times 3 \times 3} \text{ or } i=j=k=l\}$ denote the set of indices that are observed. The estimated block quadrifocal tensor $\tilde{\mathcal{Q}}^n$ is formed by stacking the blocks. Recall that quadrifocal tensors are defined up to scale and each estimated quadrifocal tensors will be associated with an unknown scale. Assuming that the true underlying block quadrifocal tensor is $\mathcal{Q}^n$, the estimated block quadrifocal tensor is $\tilde{\mathcal{Q}}^n = P_{\Omega}(\Lambda \odot_b \mathcal{Q}^n)$, where $\odot_b$ denotes block multiplication and $P_{\Omega}$ is the projection onto observed blocks. Let $S^4(\R^n)$ denote the set of symmetric tensors of size $n^4$ and let $\Lambda$ denote the variable of unknown scales. Note that by \ref{eq:qfromp}, we have the fact that
\begin{equation}
\mathcal{Q}^n_{\pi(ijkl)} = \text{sgn}(\pi) \mathcal{Q}^n_{ijkl},
\end{equation}\label{eq:symmetry}    
where $\pi$ is a permutation of the indices. Thus, we can restrict $\Lambda$ to be in the set of symmetric tensors $S^4(\R^n)$, since the elements in $Q_{\pi(ijkl)}$ will have the same value up to a sign in $Q_{ijkl}$. 

\subsection{QuadSync: An ADMM-IRLS Scheme}
We wish to solve for the scales and the camera matrices given the estimated block quadrifocal tensor $\tilde{\mathcal{Q}}^n$.   We form the following optimization problem 
\begin{align}
\underset{\Lambda,C}{\min} & \quad \sum_{(i,j,k,l)\in \Omega} \|\Lambda_{ijkl} (\tilde{\mathcal{Q}}^n)_{ijkl} - \llbracket \mathcal{G}_Q ; ~C, ~C, ~C, ~C\rrbracket_{ijkl} \|_F \nonumber \\
\text{s.t. } & \quad \Lambda \in S^4(\R^n) \text{ and } \|\Lambda\|_F^2 = 1, \nonumber
\end{align}

The condition of $\|\Lambda\|_F^2 = 1$ is enforced to avoid a trivial solution. 
The norms are added without squaring them to reduce sensitivity to outliers.
The problem is difficult to solve due to the quartic degree with respect to $C$ and the nonconvexity from the scales. 
By introducing an auxiliary variable $B\in \R^{3n\times 4}$, the problem is equivalent to the following constrained optimization problem where the factors are separated,
\begin{align} 
    \underset{\Lambda,C_i,B}{\min} & \quad \sum_{(i,j,k,l)\in \Omega} \|\Lambda_{ijkl} (\tilde{\mathcal{Q}}^n)_{ijkl} - \llbracket \mathcal{G}_Q ; ~C_1, ~C_2, ~C_3, ~C_4\rrbracket_{ijkl}) \|_F \nonumber \\
    \text{s.t. } & \quad C_1 = C_2 = C_3 = C_4 = B  \nonumber \\
     & \quad \Lambda \in S^4(\R^n) \text{ and } \|\Lambda\|_F^2 = 1. 
\end{align}
We then form the augmented Lagrangian and solve the problem through ADMM with scaled dual variables:
\begin{align*}
    \underset{\Gamma_i}{\max} & \underset{\Lambda,C_i,B}{\min} \sum_{(i,j,k,l)\in \Omega} \|\Lambda_{ijkl} (\tilde{\mathcal{Q}}^n)_{ijkl} - \llbracket \mathcal{G}_Q ;C_1,~C_2,~C_3, ~C_4\rrbracket_{ijkl}) \|_F \nonumber \\ 
    & + \frac{\rho}{2} \sum_{i=1}^4 \|C_i - B + \Gamma_i\|_F^2  \quad \text{s.t. }  \Lambda \in S^4(\R^n) \text{ and } \|\Lambda\|_F^2 = 1. 
\end{align*}

\subsubsection{IRLS Outer Loop}
To solve the $L_1$ norm optimization in the main loss function, we use Iterative Reweighted Least Squares (IRLS).

In each IRLS iteration, we calculate the current set of weights, and then solve for all the other variables. Solving for the other variables is a least squares problem. Specifically, the problem becomes 

\begin{align}\label{opt:admm_irls}
    \underset{\Gamma_i}{\max} & \underset{\Lambda,C_i,B}{\min} \|W \odot_b ( (\Lambda \odot_b \tilde{\mathcal{Q}}^n) -  \llbracket \mathcal{G}_Q;~C_1,~C_2, ~C_3, ~C_4 \rrbracket) \|^2_F + \nonumber \\ 
    &\frac{\rho}{2} \sum_{i=1}^4 \|C_i - B + \Gamma_i\|_F^2  \quad \text{s.t. } \Lambda \in S^4(\R^n) \text{ and } \|\Lambda\|_F^2 = 1. 
\end{align}
Here $W$ is the set of IRLS weights included in the norm, so
\begin{align} \label{weights} w_{ijkl} = \begin{cases}
    1 / m^t_{ijkl} \quad \text{ if } (i,j,k,l)\in \Omega\\
    0 \quad \quad \quad \quad \quad \text{otherwise}
\end{cases} \end{align}
and 
\begin{align*}
    m^t_{ijkl} = \max(\delta,~sqrt(\|\Lambda^{(t-1)}_{ijkl} \tilde{\mathcal{Q}}^n_{ijkl} - \G_Q \times_1 (C_1)^{(t-1)}_i \times_2 (C_2)^{(t-1)}_j \times_3 (C_3)^{(t-1)}_k \times_4 (C_4)^{(t-1)}_l)\|_F).
\end{align*}
using the variables from the previous IRLS iteration. $\delta$ is a regularization parameter that upper-bounds the weights. 

\subsubsection{ADMM Inner Loop}
After fixing the weights for the outer IRLS loop, we solve the weighted least squares in (\ref{opt:admm_irls}) for $\Gamma_i, \Lambda, C_i, B$ with Alternating Direction Method of Multipliers (ADMM).   
\begin{enumerate}
    \item[(1)] $\mathbf{C_1,C_2,C_3,C_4,\Lambda :}$ We first solve for $\Lambda$ and the factors $C_i$ through alternating minimization. 
    To solve for $C_i$'s, given the convexity of the problem when fixing other variables, we can explicitly calculate the derivative and solve for $C_i$ by setting it to zero. Note that the flattening of the tensor can be written in terms of the factor matrices and the core from the tucker factorization. For example, $\mathcal{Q}^n_{(1)} = C_1 (\mathcal{G}_Q)_{(1)} (C_4 \otimes C_3 \otimes C_2)^T$. We use $C_{\sim i}$ to denote the matrix from kronecker products, so that for the $i$th mode, $C_{\sim i} = C_4 \otimes \cdots \otimes C_{i+1 }\otimes C_{i-1} \otimes \cdots \otimes C_1$. Let $K = (\G_Q)_{(i)} C_{\sim i}^T $. Then,  
    
    \begin{align*}
       \frac{\rho}{2} C_i + \left( W_{(i)}^2 \odot_b \left[ C_i K \right] \right) K^T =  \frac{\rho}{2}(B-\Gamma_i) + \left( W^2_{(i)} \odot_b (\Lambda \odot_b \tilde{\mathcal{Q}}^n)_{(i)}\right) K^T.
    \end{align*}
    We can solve for each row of $C_i$ separately.  Letting $x_j$ denote the $j$th row of $C_i$, we have 
    \begin{align*}
         x_j ( \frac{\rho}{2}  I_{4\times 4} + K \text{diag}((W_{(i)}^2)_j) K^T) = \frac{\rho}{2}(B-\Gamma_i)_j + ((W_{(i)}^2)_j \odot_b [(\Lambda \odot_b \tilde{\mathcal{Q}}^n)_{(i)}]_{j}) K^T
    \end{align*}
    and 
    \begin{align}\label{Cupdate}
        x_j  = (\frac{\rho}{2}(B-\Gamma_i)_j + ((W^2_{(i)})_j \odot_b [(\Lambda \odot_b \tilde{\mathcal{Q}}^n)_{(i)}]_{j}) K^T)( \frac{\rho}{2}I_{4\times 4} + K \text{diag}((W^2_{(i)})_j) K^T)^{-1}
    \end{align}
    where $\text{diag}((W^2_{(i)})_j)$ is the diagonal matrix corresponding to $((W\otimes 1_{3\times 3\times 3 \times 3})^2_{(i)})_j $. 
    \begin{remark}
        Solving for $C_i$ via \eqref{Cupdate} can be slow, as the dimensions for $K$ are $4 \times 27n^3$, and $\text{diag}(W_j^2)$ is $27n^3 \times 27n^3$. Though it is sparse, it can become computationally expensive. Each row can be solved in parallel to speed up computations. In addition, one can easily apply randomized updates to speed up the algorithm. One can randomly sample $m$ columns, so that the multiplication becomes of size $4\times m$ and $m\times m$ and $m\times 4$. Since the low-rankness is independent of the number of cameras, we can take $m=O(1)$. We include a small experiment in the appendices, showing that speed-ups can be achieved without sacrificing accuracy this way. 
    \end{remark}
    
    For $\Lambda$, we can solve by directly solving the convex optimization problem: 
    \begin{align}\label{Lupdate}
        \underset{\Lambda}{\min} & \quad \|W \odot_b [(\Lambda \odot_b \tilde{\mathcal{Q}}^n) -  \llbracket \mathcal{G}_Q ; ~C_1, ~C_2,~ C_3, ~C_4 \rrbracket]\|_F^2  \nonumber \\
        \text{s.t. } & \quad \quad \Lambda \in S^4(\R^n) \text{ and } \|\Lambda\|_F^2 = 1.
    \end{align}
    This is done by solving for each block separately, where 
    \begin{align*}
    \Lambda_{ijkl} = \frac{trace((\llbracket \mathcal{G}_Q ;~ C_1,~ C_2,~ C_3,~ C_4 \rrbracket_{ijkl})_{(1)}^T ((\tilde{\mathcal{Q}}^n)_{ijkl}))_{(1)})}{\|((\tilde{\mathcal{Q}}^n)_{ijkl})_{(1)}\|_F^2}.
    \end{align*}
    After calculating $\Lambda$, we first symmetrize, then normalize $\Lambda$ so that $\|\Lambda\|_F^2 = 1$ and $\Lambda \in S^4(\R^n)$. We iterate this process for solving for $C_i,\Lambda$ iteratively for 10 times. 
    
    \item[(2)] $\mathbf{B:}$ For $B$, we can solve directly, via 
    \begin{equation} \label{Bupdate}
        \frac{1}{4} \left( \sum_{i=1}^4 C_i + \Gamma_i \right) = B.
    \end{equation}
    \item[(3)] $\mathbf{\Gamma_i:}$ For the final ascent step, we set \begin{equation}\label{Gupdate}
    \Gamma_{i}^{(k+1)} = \Gamma_i^{(k)} + (C_i - B).
    \end{equation}
\end{enumerate}

\subsection{A Joint Optimization Framework}

We also develop a joint optimization framework that simultaneously synchronizes the block quadrifocal tensor, the block trifocal tensor, and the block essential matrix. We develop it only for the calibrated case. 
We observe that the block quadrifocal tensor and the block trifocal tensor share same factor matrices, i.e., the $3n \times 4$ stacked camera matrices. Let $\mathcal{P}$ denote the $3n\times 6$ matrix of the line projection matrices associated with each camera matrix. See the appendices for more details on the line projection matrices. We establish a new factorization of the block essential matrix, and show that it shares factor matrices with the block trifocal tensor in the first mode in the following result. 
\begin{theorem}\label{thrm:ess decomposition}
    Let $\mathcal{E}^n \in \mathcal{S}^2(\R^{3n})$ be a block essential matrix corresponding to the calibrated cameras $C\in \R^{3n\times 4}$.  Then $\mathcal{E}^n$ admits the factorization:
    \begin{equation}
        \mathcal{E}^n = \mathcal{P} \begin{pmatrix}
            0 & 0 & 0 & 0 & 0 & 1\\
            0 & 0 & 0 & 0 & 1 & 0\\
            0 & 0 & 0 & -1 & 0 & 0\\
            0 & 0 & -1 & 0 & 0 & 0\\
            0 & 1 & 0 & 0 & 0 & 0\\
            -1 & 0 & 0 & 0 & 0 & 0
        \end{pmatrix} \mathcal{P}^T = \mathcal{P} \mathcal{G}_E\mathcal{P}^T.
    \end{equation}
\end{theorem}
We next establish the projection rank for the block trifocal tensor, which will also be implicitly enforced. 
\begin{theorem}\label{thrm:tft p rank}
    Let $\mathcal{T}^n \in \mathbb{R}^{3n \times 3n \times 3n}$ be a block trifocal tensor. Then P-Rank($\mathcal{T}^n$) = (4,~3,~3). 
\end{theorem}
We denote the core matrix here as $\mathcal{G}_E$. Since the block trifocal tensor $\mathcal{T}^n$ and block quadrifocal tensor $\mathcal{Q}^n$ can be factorized as $\mathcal{T}^n = \llbracket \mathcal{G}_T; \mathcal{P}, C, C \rrbracket$ and $ \mathcal{Q}^n= \llbracket \mathcal{G}_Q; C, C, C, C\rrbracket$, we can formulate the following optimization problem to synchronize all three entities simultaneously: 
\begin{align*}
    \underset{\Lambda_E,\Lambda_T,\Lambda_Q,C}{\min} & \quad \frac{1}{n_Q}\|W_Q \odot_b (\Lambda_Q \odot_b \tilde{\mathcal{Q}}^n - \llbracket \G_Q ; ~C, ~C, ~C, ~C\rrbracket \|_F^2 + \\
    & \quad \frac{1}{n_T}\|W_T \odot_b (\Lambda_T \odot_b \tilde{\mathcal{T}}^n - \llbracket \G_T ; ~\mathcal{P}, ~C, ~C \rrbracket)\|_F^2 + \\
    & \quad \frac{1}{n_E}\|W_E \odot_b (\Lambda_E \odot_b \tilde{\mathcal{E}}^n - \llbracket \G_E ; ~\mathcal{P}, ~\mathcal{P} \rrbracket\|_F^2 \\
    \text{s.t. } & \quad  \|\Lambda_Q\|_F^2 = \|\Lambda_T\|_F^2 = \|\Lambda_E\|_F^2 = 1, \\
    & \quad  \Lambda_Q \in S^4(\R^n), ~ \Lambda_T(i,:,:), ~\Lambda_E \in {S}^2(\R^{n}),
\end{align*}
where $n_Q, n_T, n_E$ are the number of estimated blocks in the estimated block quadrifocal tensor $\tilde{\mathcal{Q}}^n$, the estimated block trifocal tensor $\tilde{\mathcal{T}}^n$, and the estimated block essential matrix $\tilde{\mathcal{E}}^n$, respectively. After separating factors, let us denote the three cost functions as $f_{Q}(W_Q, ~\Lambda_Q,~ C_1,~C_2,~C_3,~C_4)$, $g_T(W_T, ~\Lambda_T,~P_1 ~C_5,~C_6)$, $h_E(W_E,~\Lambda_E,~P_2,~P_3)$, respectively. We also have IRLS weights set  for $W_q,~W_t,~W_e$ defined analogously to \eqref{weights}. See the appendices for more. 

\subsubsection{Joint Opt. with ADMM-IRLS}
Similarly, we solve the joint optimization problem with an ADMM formulation:
\begin{align*}
    & \underset{\Lambda_E,\Lambda_T,\Lambda_Q,C}{\min}  f_{Q}(W_Q,~\Lambda_Q,~ C_1,~C_2,~C_3,~C_4)+ g_T(W_T, ~\Lambda_T,~P_1,~C_5,~C_6)+ h_E(W_E,~\Lambda_E,~P_2,~P_3) + \\
    & \quad\quad\quad\quad \frac{\rho}{2}\sum_{i=1}^6 \|C_i - B + \Gamma_i\|_F^2  + \frac{\rho}{2}\sum_{i=1}^3 \|P_i - D + \tau_i\|_F^2
\end{align*}
 such that $\|\Lambda_Q\|_F^2 = \|\Lambda_T\|_F^2 = \|\Lambda_E\|_F^2 = 1,  \Lambda_Q \in S^4(\R^n)$, $\Lambda_T(i,:,:)$ and  $\Lambda_E \in {S}^2(\R^{n})$.
\begin{enumerate}
    \item[(1)] $\mathbf{C_i}$, $\mathbf{P_i}$: We first aim to solve for the $C_i$'s and $P_i's$. The updates for $C_1,C_2,C_3,C_4$ are the same as before in the quadrifocal tensor case but with additional weights $n_Q$. For $C_5,C_6$, the derivation is identical. To solve for $C_i$, let $C_{\sim i} = P \otimes C_j$, and let $K = (\G_T)_{(i)} C_{\sim i}^T $. Then, we have the update rule for $C_i$ as 
\begin{align*}
    x_j = (\frac{\rho}{2}( B -\Gamma_i)_j + \frac{1}{n_T}((W_T^2 \odot L_T) \odot_b T)_j K^T)(\frac{\rho}{2}I_{4\times 4} + \frac{1}{n_T} K \text{diag}((W_T)_j^2)K^T)^{-1}.
\end{align*}
For $P_1$, let $K = (\G_T)_{(3)}(C_2 \otimes C_1)^T$. Then
\begin{align*}
    x_j = (\frac{\rho}{2}( D -\tau)_j + \frac{1}{n_T}((W_T^2 \odot L_T) \odot_b T)_j K^T) (\frac{\rho}{2}I_{4\times 4} + \frac{1}{n_T} K \text{diag}(W_T^2)K^T)^{-1}.
\end{align*}
Similarly, we also solve for each row separately for $P_2,P_3$ with $\Lambda_E, W_E, \G_E$ and $n_E$.

\item[(2)] $\mathbf{\Lambda_Q}$, $\mathbf{\Lambda_T}$, $\mathbf{\Lambda_E}$, $\mathbf{B}$, $\mathbf{D}$, $\mathbf{\Gamma_i}$, $\mathbf{\tau_i}$: They are all the same or analogous as in (\ref{Lupdate}), (\ref{Bupdate}), (\ref{Gupdate}). The order of the optimization is $C_1$, $C_2$, $C_3$, $C_4$, $C_5$, $C_6$, $P_1$, $P_2$, $P_3$, $\Lambda_Q$, $\Lambda_T$, $\Lambda_E$ alternatively for a couple iterations. Then $B,D$, then $\Gamma_i, \tau_i$. After the solving the ADMM step, we recalculate the IRLS weights and repeat until convergence. More details of the algorithm can be found in the appendices.  
\end{enumerate}

\section{Experiments}\label{sec:experiments}
\subsection{A Higher-Order Cycle Based Heuristic}\label{sec:heuristic}

We develop a heuristic to estimate the corruption level of a quadruplet of cameras using trifocal tensors. 
Similar consistency measurements for group synchronization have been proposed in \cite{duncan2025higher}, but here our objects are camera matrices and do not admit a group structure since the matrices are non-square.

Given a quadruplet of images $I_i,~I_j,~I_k,~I_l$, suppose we have four trifocal tensor estimates $T_{ijk}, ~T_{jkl}, ~T_{kli}, ~T_{lij}$. Assuming each trifocal tensor is consistent, we have camera matrices corresponding to each trifocal tensor. Specifically, suppose $T_{ijk}$ gives $\{P^1_i,~P^1_j,~P^1_k\}$, $T_{jkl}$ gives $\{P^2_j,~P^2_k,~P^2_l\}$, $T_{kli}$ gives $\{P^3_k,~P^3_l,~P^3_i\}$, and $T_{lij}$ gives $\{P^4_l,~P^4_i,~P^4_j\}$, where each set of cameras is in a different projective frame.
Given two sets of the same cameras $P^1_j,~P^1_k$ and $P^2_j, ~P^2_k$ in different projective frames, they can be synchronized by solving for a projective transformation $H_{21} \in \R^{4\times 4}$ such that $P^2_j H_{21} = a_jP^1_j$ and $P^2_k H_{21} = a_kP^1_k$. Then, $\{P^2_jH_{21},~P^2_kH_{21},~P^2_lH_{21}\}$ is in the same frame as $\{P^1_i,~P^1_j,~P^1_k\}$. This can be solved for via linear least squares in the variables $H_{21}, ~a_j, ~a_k$. 
Denote this operation by $\{P^1_i,~P^1_j,~P^1_k\} \xleftarrow{H_{21}} \{P^2_j,~P^2_k,~P^2_l\}$. 
The following operations can then be performed:
\begin{align*}
    \{P^1_i,~P^1_j,~P^1_k\} \xleftarrow{H_{21}}& \{P^2_j,~P^2_k,~P^2_l\} \\
    \{P^2_jH_{21},~P^2_kH_{21},~P^2_lH_{21}\} \xleftarrow{H_{32}}& \{P^3_k,~P^3_l,~P^3_i\} \\
    \{P^3_kH_{32},~P^3_lH_{32},~P^3_iH_{32}\} \xleftarrow{H_{43}}& \{P^4_l,~P^4_i,~P^4_j\}.
\end{align*}
When there is no noise in any of the initial trifocal tensors, $a P^1_i = P^4_i H_{43}$ and $ b P^1_j = P^4_jH_{43}$ for some nonzero scales $a,b \in \mathbb{R}$. However, when there is noise, this will not hold. Our heuristic $d(i,~j,~k,~l)$ can be defined as a tuple measuring the distance between the first two and last two cameras: 
\begin{align}\label{heuristic} 
    d(i,~j,~k,~l) := g(d(P^1_i,~P^4_iH_{43}), \hspace{0.2em} ~d(P^1_j,~P^4_jH_{43})). 
\end{align}
In our experiments, we choose $d(P_m,~P_n)$ as either the angle between the orientation of calibrated cameras $P_m,~P_n$ or the relative difference of the locations. 
Further, we choose $g(\cdot, \cdot)$ to be the average. 

\begin{figure*}
    \centering
    \includegraphics[width=\linewidth,height=7.2cm]{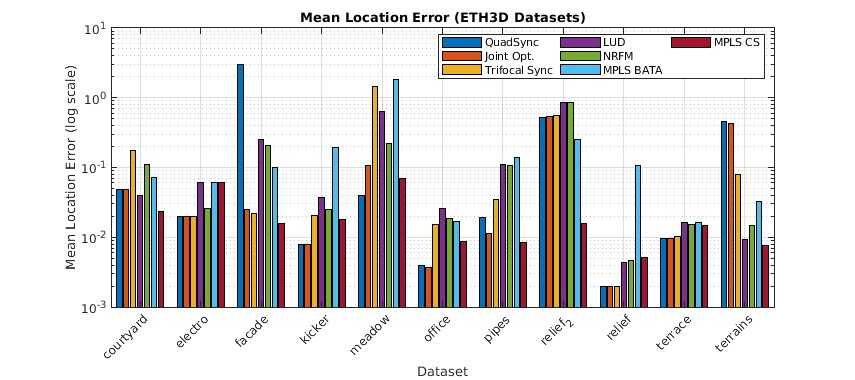}
    \caption{Mean location error for ETH3D datasets}
    \label{fig:Mean Location Error ETH3D}
\end{figure*}

\begin{figure*}
    \centering
    \includegraphics[width=\linewidth, height = 7.2cm]{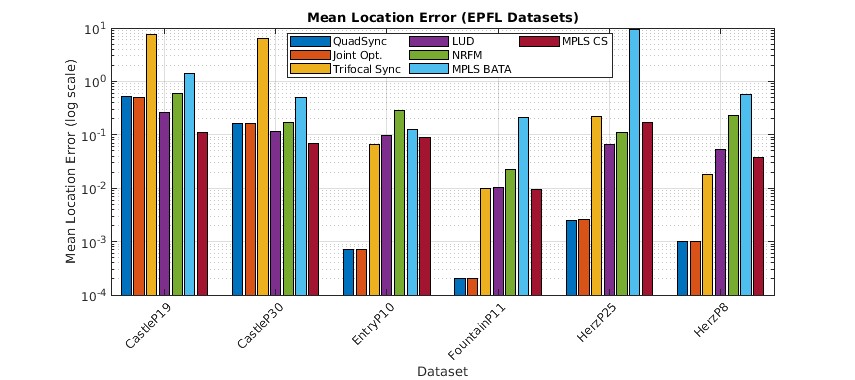}
    \caption{Mean location error for EPFL datasets}
    \label{fig:Mean Location Error EPFL}
\end{figure*}
\subsection{Estimation of Quadrifocal Tensors}\label{sec: qft est} 

Given images $I_i$ with $i=1,...,N$, the procedure of estimating the quadrifocal tensors is described below. 

\begin{enumerate}
    \item[\textbf{1.}] First, estimate trifocal tensors from point correspondences. We follow the procedure in \cite{miao2024tensor}, where the Subspace Constrained Tyler's M-estimator (STE) from \cite{yu2024subspace} is used to reject outliers, assuming at most 40\% of point correspondences are inliers. Initial trifocal tensors are estimated linearly with at most 30 inlier points and refined with bundle adjustment. Note that the corresponding trifocal tensor $T_{ijk}$ can also be estimated using the essential matrices $E_{ji}, ~E_{ki},~ E_{kj}$. A detailed description of the latter method can be found in \cite{hartley2003multiple}. 
    
    \item[\textbf{2.}] Then, enumerate all four cycles in this viewing graph and calculate the inconsistency heuristic \eqref{heuristic}. Cycles whose heuristic in terms of rotation is larger than 3 degrees or is larger than 0.2 in terms of translation are neglected. For good cycles, we simply take an average of all the camera pose estimates obtained by calculating the heuristic. The quadrifocal tensor is then estimated via \eqref{eq:qfromp}.  
    
\end{enumerate}
After estimation of the quadrifocal tensors by the above procedure, we form the estimated block quadrifocal tensor $\tilde{\mathcal{Q}}^n$. We also fill in the blocks with two overlapping indices, which correspond to the trifocal tensors. We normalize each block in $\tilde{\mathcal{Q}}^n$ to have Frobenius norm 1. 
\begin{remark} There is a lack of practical work for estimating quadrifocal tensors. The current procedure for estimating quadrifocal tensors introduces extra sources of error. Yet, numerical results show resilience and robustness against these extra sources of noise. Our work also motivates future research on the estimation of quadrifocal tensors.
\end{remark}

\subsection{Initialization of Algorithm}
For QuadSync, we initialize the following variables $W_Q, ~\Lambda_Q, ~C_i, ~B, ~\Gamma_i$. We set $\rho = 0.01$, run IRLS for 4 times, and use just 1 ADMM loop inside each IRLS.  
\begin{itemize}
    \item $\mathbf{C_i}$, $\mathbf{B}$: For the initialization of the camera matrices, we first retrieve the camera matrices from $\tilde{\mathcal{Q}}^n$. That is, we apply the higher order singular value decomposition and simply take the first four singular vectors from any mode as the initial estimate for the cameras $\tilde{C}$. We set $C_1 = C_2 = C_3 = C_4 = B = \tilde{C}$. 
    \item $\mathbf{\Lambda_Q}$, $\mathbf{W_Q}$: We solve \eqref{Lupdate} for the optimal set of scales. We then initialize $W_Q$ with \eqref{weights}. 
    \item $\mathbf{\Gamma_i}$: We set $\Gamma_i^0 = 0_{3n\times 4}$ for all $i\in \{1,2,3,4\}$. 
\end{itemize} 
For the joint optimization, the variables are all initialized similarly, but $\rho$ is set to $0.00001$, and we only run 2 IRLS loops. 
\begin{remark} Though our optimization problems are nonconvex, an initialization with HOSVD appears to be sufficient based on empirical evidence. This is an advantage, as two view methods such as \cite{sengupta2017new}, which also solve complicated nonconvex problems, usually rely on initializing with specialized and more complex methods. 
\end{remark}

\subsection{Numerical Experiments}
Our method operates on datasets that are dense, such that forming the higher order 4-uniform viewing hypergraphs does not disconnect the viewing graph or cause the viewing graph to become too sparse.  We truncate all datasets so that they satisfy a density requirement, as follows.  We first obtain all the indices $(i,j,k,l) \in \Omega$. 
Then we calculate the density of each vertex in all observed quadruplet indices. 
We delete any vertex where the density is lower than 0.05. 
If that results in a viewing graph with too few cameras, the threshold is lowered to 0.02 and then 0.01. 
One of the datasets in ETH3D does not provide a dense enough graph and is discarded. To compare with the joint optimization framework and other global synchronization methods, we restrict our algorithm to the calibrated setting. We compare with the following methods: TrifocalSync \cite{miao2024tensor}, NRFM \cite{sengupta2017new}, MPLS \cite{shi2020message}, LUD \cite{ozyesil2015stable},  BATA \cite{zhuang2018baseline}, and a very recent SOTA method Cycle-Sync~\cite{li2025cycle}. Note that BATA and Cycle-Sync are location synchronization algorithms.  We use MPLS for their rotation synchronization component.

\subsubsection{ETH3D}
We apply our methods to 11 stereo multiview benchmark datasets in ETH3D from \cite{schoeps2017cvpr,Schops_2019_CVPR}. ETH3D is a diverse collection of datasets, ranging from artificial to natural as well as indoor to outdoor environments. It has been used widely as benchmarks for SfM. We preprocess the datasets such that if the two view viewing graph contains weakly connected components, we retain only the largest component. We estimate the keypoint matches using SIFT feature points. We then estimate the trifocal tensors with the implementations in \cite{julia2017critical}. We construct the block quadrifocal tensor following the procedure in Section \ref{sec: qft est}. We also estimate the essential matrices using RANSAC for comparison purposes.

\subsubsection{EPFL}
We also apply our methods to 6 multiview stereo high resolution datasets in EPFL \cite{strecha2008benchmarking}.
We use GlueStick \cite{pautrat2023gluestick} to detect and match feature points. We then estimate the trifocal tensors following the same procedure as with ETH3D. We estimate the fundamental matrices with GC-RANSAC \cite{barath2018graph} for comparison purposes. We then form the estimated block quadrifocal tensor.  

\subsubsection{Results}
The mean location errors for the ETH3D and EPFL datasets are displayed in Figures \ref{fig:Mean Location Error ETH3D} and \ref{fig:Mean Location Error EPFL}, respectively. We also report the median location errors, the mean rotation errors, and the median rotation errors, and refer to tables in the appendices for comprehensive results. In terms of location quality, we see that QuadSync and Joint Opt. perform the best or very close to the best in 7/11 of the datasets in ETH3D, and 4/6 of the EPFL datasets. Since our methods operate on higher-order measurements, they prefer dense viewing graphs. Among most of the datasets on which we perform poorly, the ratio of the number of estimated quadrifocal blocks to the total number of blocks drops below $30\%$. However, among the datasets where this ratio is larger than $70\%$, our proposed methods tend to find better quality solutions than SOTA algorithms. Having clean estimations and a high completion rate are important for the success of the algorithm. The completion rate and runtime of our algorithm can be found in the appendix. 
\begin{remark} Since we operate on dense graphs which are computationally heavy, our methods may perform best with distributed synchronization approaches. Two small experiments to demonstrate the potential of applying the methods to large scale datasets, when given good clusters, is included in the appendices. We hope our work may motivate more research on distributed approaches for higher-order methods. As mentioned in Remark~\ref{remark: collinear}, another advantage of using quadrifocal tensors is that they become less sensitive in the case of collinear configurations. We experiment on a near-collinear subsequence of views from the plant\_scene\_1 dataset from ETH3D SLAM, see Figure~\ref{fig:Near_Collinear} for the reconstructed camera poses. Note that this configuration can't be reconstructed with the pairwise measurements of fundamental matrices. 
\end{remark}

\begin{figure}[htp!]
    \centering
    \includegraphics[width=0.5\linewidth]{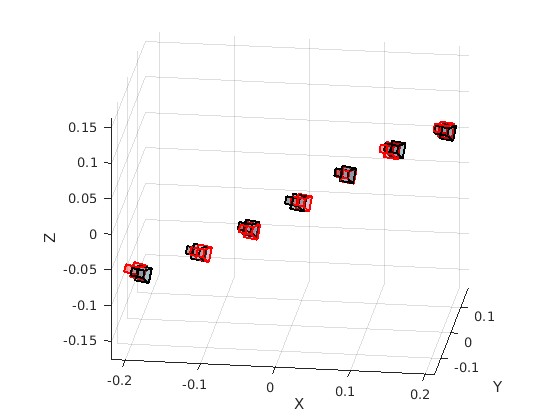}
    \caption{QuadSync retrieved camera poses on near-collinear views from plant\_scene\_1 dataset from ETH3D SLAM.}
    \label{fig:Near_Collinear}
\end{figure}

Another small synthetic experiment to show the effectiveness of our theory and algorithm for collinear configurations of cameras is included in the appendices. We see that the algorithm is insensitive to the collinear configuration and can still successfully synchronize the quadrifocal tensors.

\section{Conclusion}\label{sec:conclusion}

We investigated totally novel approaches to synchronization in SfM.  We introduced the block quadrifocal tensor, and laid  theoretical foundations for its synchronization.  We established an explicit Tucker decomposition for the block quadrifocal tensor, in which the factor matrices are exactly the camera poses. We established the low multilinear rank of $(4,~4,~4,~4)$ and low projection rank of $(2,~2,~2,~2,~2,~2)$. We showed that the multilinear rank is sufficient to determine the scales for the block quadrifocal tensor, enabling camera pose retrieval. We also established additional new properties of the $n$-view essential matrix and the block trifocal tensor. We devised the algorithms QuadSync and Joint Opt. for the synchronization of the block quadrifocal tensor, and for the joint synchronization of all three entities together. Experiments showed the methods can improve accuracy in synchronization, especially location accuracy. 
 In sum, we provided the first ever work on synchronization through quadrifocal tensors, and showed that there is great promise for advancing SfM systems via higher-order measurements.

There are many opportunities for future work.  Though we develop strong theory, the algorithms depend on the quality of the block estimations. This motivates work on estimating and processing quadrifocal tensors. Our algorithms also require dense subsets and are more computationally demanding than two view and three view methods. The development of distributed synchronization methods for higher-order methods is needed. 

\section*{Acknowledgement}
D.M.~and G.L.~were supported in part by NSF award DMS 2152766. D.M.~was also supported by the DSI-MnDRIVE PhD Graduate Assistantship from the University of Minnesota Data Science Initiative and Minnesota’s Discovery, Research, and Innovation Economy. J.K.~was supported in part by NSF awards DMS 2309782 and CISE-IIS 2312746, the DOE award SC0025312, the Sloan Foundation, and start-up grants from the College of Natural Science and Oden Institute at the University of Texas at Austin.

We thank Shaohan Li and Michael Thomas for helpful discussions on processing ETH3D datasets. We thank Yifan Zhang for suggestions on tuning ADMM algorithms with tensors. 

\bibliography{main}

@String(IJCV = {Int. J. Comput. Vis.})

@String(CVPR= {IEEE Conf. Comput. Vis. Pattern Recog.})

@String(ICCV= {Int. Conf. Comput. Vis.})

@String(ECCV= {Eur. Conf. Comput. Vis.})

@String(NIPS= {Adv. Neural Inform. Process. Syst.})

@String(ACMMM= {ACM Int. Conf. Multimedia})

@String(IJCV  = {IJCV})

@String(CVPR  = {CVPR})

@String(ICCV  = {ICCV})

@String(ECCV  = {ECCV})

@String(NIPS  = {NeurIPS})

@String(ACMMM = {ACM MM})

@book{hartley2003multiple,
  title={Multiple {V}iew {G}eometry in {C}omputer {V}ision},
  author={Hartley, Richard and Zisserman, Andrew},
  year={2003},
  publisher={Cambridge university press}
}

@inproceedings{kasten2019gpsfm,
  title={G{PSfM}: Global projective sfm using algebraic constraints on multi-view fundamental matrices},
  author={Kasten, Yoni and Geifman, Amnon and Galun, Meirav and Basri, Ronen},
  booktitle=CVPR,
  pages={3264--3272},
  year={2019}
}

@inproceedings{sengupta2017new,
  title={A new rank constraint on multi-view fundamental matrices, and its application to camera location recovery},
  author={Sengupta, Soumyadip and Amir, Tal and Galun, Meirav and Goldstein, Tom and Jacobs, David W and Singer, Amit and Basri, Ronen},
  booktitle={CVPR},
  pages={4798--4806},
  year={2017}
}

@inproceedings{kasten2019algebraic,
  title={Algebraic characterization of essential matrices and their averaging in multiview settings},
  author={Kasten, Yoni and Geifman, Amnon and Galun, Meirav and Basri, Ronen},
  booktitle={ICCV},
  pages={5895--5903},
  year={2019}
}

@inproceedings{geifman2020averaging,
  title={Averaging essential and fundamental matrices in collinear camera settings},
  author={Geifman, Amnon and Kasten, Yoni and Galun, Meirav and Basri, Ronen},
  booktitle={CVPR},
  pages={13--19},
  year={2020}
}

@inproceedings{miao2024tensor,
  title={Tensor-based synchronization and the low-rankness of the block trifocal tensor},
  author={Miao, Daniel and Lerman, Gilad and Kileel, Joe},
  booktitle=NIPS,
  volume={37},
  pages={69505--69532},
  year={2024}
}

@inproceedings{shashua2000structure,
  title={On the structure and properties of the quadrifocal tensor},
  author={Shashua, Amnon and Wolf, Lior},
  booktitle=ECCV,
  pages={710--724},
  year={2000}
}

@inproceedings{hartley1998computation,
  title={Computation of the quadrifocal tensor},
  author={Hartley, Richard I},
  booktitle={ECCV},
  pages={20--35},
  year={1998},
  organization={Springer}
}

@article{oeding2017quadrifocal,
  title={The quadrifocal variety},
  author={Oeding, Luke},
  journal={Linear Algebra and its Applications},
  volume={512},
  pages={306--330},
  year={2017},
  publisher={Elsevier}
}

@article{duncan2025higher,
  title={Higher-Order Group Synchronization},
  author={Duncan, Adriana L and Kileel, Joe},
  journal={arXiv preprint arXiv:2505.21932},
  year={2025}
}

@inproceedings{yu2024subspace,
  title={A subspace-constrained {T}yler's estimator and its applications to structure from motion},
  author={Yu, Feng and Zhang, Teng and Lerman, Gilad},
  booktitle={CVPR},
  pages={14575--14584},
  year={2024}
}

@inproceedings{wilson2014robust,
  title={Robust global translations with 1{DSFM}},
  author={Wilson, Kyle and Snavely, Noah},
  booktitle={ECCV},
  pages={61--75},
  year={2014},
  organization={Springer}
}

@article{snapshot2025,
  title={Snapshot of algebraic vision},
  author={Kileel, Joe and Kohn, Kathlen},
  journal={AMS Proceedings of Symposia in Pure Mathematics},
  volume={111},
  pages={323--363},
  year={2025}
}

@article{li2025cycle,
  title={Cycle-{S}ync: Robust Global Camera Pose Estimation through Enhanced Cycle-Consistent Synchronization},
  author={Li, Shaohan and Shi, Yunpeng and Lerman, Gilad},
  journal={arXiv preprint arXiv:2511.02329},
  year={2025}
}

@article{ozyesil2015stable,
  title={Stable camera motion estimation using convex programming},
  author={Ozyesil, Onur and Singer, Amit and Basri, Ronen},
  journal={SIAM Journal on Imaging Sciences},
  volume={8},
  number={2},
  pages={1220--1262},
  year={2015},
  publisher={SIAM}
}

@inproceedings{shi2020message,
  title={Message passing least squares framework and its application to rotation synchronization},
  author={Shi, Yunpeng and Lerman, Gilad},
  booktitle={ICML},
  pages={8796--8806},
  year={2020},
  organization={PMLR}
}

@inproceedings{zhuang2018baseline,
  title={Baseline desensitizing in translation averaging},
  author={Zhuang, Bingbing and Cheong, Loong-Fah and Lee, Gim Hee},
  booktitle={CVPR},
  pages={4539--4547},
  year={2018}
}

@inproceedings{pautrat2023gluestick,
  title={Glue{S}tick: Robust image matching by sticking points and lines together},
  author={Pautrat, R{\'e}mi and Su{\'a}rez, Iago and Yu, Yifan and Pollefeys, Marc and Larsson, Viktor},
  booktitle=ICCV,
  pages={9706--9716},
  year={2023}
}

@inproceedings{barath2018graph,
  title={Graph-cut {RANSAC}},
  author={Barath, Daniel and Matas, Ji{\v{r}}{\'\i}},
  booktitle={CVPR},
  pages={6733--6741},
  year={2018}
}

@article{aholt2014ideal,
  title={The ideal of the trifocal variety},
  author={Aholt, Chris and Oeding, Luke},
  journal={Mathematics of Computation},
  volume={83},
  number={289},
  pages={2553--2574},
  year={2014}
}

@article{kolda2009tensor,
  title={Tensor decompositions and applications},
  author={Kolda, Tamara G and Bader, Brett W},
  journal={SIAM Review},
  volume={51},
  number={3},
  pages={455--500},
  year={2009},
  publisher={SIAM}
}

@inproceedings{schoeps2017cvpr,
  author = {Thomas Sch\"ops and Johannes L. Sch\"onberger and Silvano Galliani and Torsten Sattler and Konrad Schindler and Marc Pollefeys and Andreas Geiger},
  title = {A Multi-View Stereo Benchmark with High-Resolution Images and Multi-Camera Videos},
  booktitle = CVPR,
  pages = {3260-3269},
  year = {2017}
}

@InProceedings{Schops_2019_CVPR,
author = {Sch\"ops, Thomas and Sattler, Torsten and Pollefeys, Marc},
title = {{BAD SLAM}: Bundle Adjusted Direct {RGB-D SLAM}},
booktitle = CVPR,
pages = {134--144},
year = {2019}
}

@inproceedings{julia2017critical,
  title={A critical review of the trifocal tensor estimation},
  author={Juli{\`a}, Laura F and Monasse, Pascal},
  booktitle={Pacific-Rim Symposium on Image and Video Technology},
  pages={337--349},
  year={2017},
  organization={Springer}
}

@inproceedings{strecha2008benchmarking,
  title={On benchmarking camera calibration and multi-view stereo for high resolution imagery},
  author={Strecha, Christoph and Von Hansen, Wolfgang and Van Gool, Luc and Fua, Pascal and Thoennessen, Ulrich},
  booktitle=CVPR,
  pages={1--8},
  year={2008}
}

@incollection{snavely2006photo,
  title={Photo {T}ourism: {E}xploring photo collections in 3{D}},
  author={Snavely, Noah and Seitz, Steven M and Szeliski, Richard},
  booktitle={SIGGRAPH},
  pages={835--846},
  year={2006}
}

@inproceedings{schonberger2016structure,
  title={Structure-from-motion revisited},
  author={Schonberger, Johannes L and Frahm, Jan-Michael},
  booktitle={CVPR},
  pages={4104--4113},
  year={2016}
}

@inproceedings{madhavan2025recovery,
  title={On the Recovery of Cameras from Fundamental Matrices},
  author={Madhavan, Rakshith and Arrigoni, Federica},
  booktitle={ICCV},
  pages={20934--20943},
  year={2025}
}

@inproceedings{govindu2004lie,
  title={Lie-algebraic averaging for globally consistent motion estimation},
  author={Govindu, Venu Madhav},
  booktitle=CVPR,
  volume={1},
  pages={1--8},
  year={2004}
}

@article{hartley2013rotation,
  title={Rotation averaging},
  author={Hartley, Richard and Trumpf, Jochen and Dai, Yuchao and Li, Hongdong},
  journal=IJCV,
  volume={103},
  number={3},
  pages={267--305},
  year={2013},
  publisher={Springer}
}

@article{chatterjee2017robust,
  title={Robust relative rotation averaging},
  author={Chatterjee, Avishek and Govindu, Venu Madhav},
  journal={IEEE Transactions on Pattern Analysis and Machine Intelligence},
  volume={40},
  number={4},
  pages={958--972},
  year={2017},
  publisher={IEEE}
}

@inproceedings{chatterjee2013efficient,
  title={Efficient and robust large-scale rotation averaging},
  author={Chatterjee, Avishek and Govindu, Venu Madhav},
  booktitle=ICCV,
  pages={521--528},
  year={2013}
}

@inproceedings{arie2012global,
  title={Global motion estimation from point matches},
  author={Arie-Nachimson, Mica and Kovalsky, Shahar Z and Kemelmacher-Shlizerman, Ira and Singer, Amit and Basri, Ronen},
  booktitle={2012 Second International Conference on 3D Imaging, Modeling, Processing, Visualization \& Transmission},
  pages={81--88},
  year={2012},
  organization={IEEE}
}

@inproceedings{goldstein2016shapefit,
  title={Shapefit and {S}hapekick for robust, scalable structure from motion},
  author={Goldstein, Thomas and Hand, Paul and Lee, Choongbum and Voroninski, Vladislav and Soatto, Stefano},
  booktitle=ECCV,
  pages={289--304},
  year={2016},
  organization={Springer}
}

@article{arrigoni2016spectral,
  title={Spectral synchronization of multiple views in {SE} (3)},
  author={Arrigoni, Federica and Rossi, Beatrice and Fusiello, Andrea},
  journal={SIAM Journal on Imaging Sciences},
  volume={9},
  number={4},
  pages={1963--1990},
  year={2016},
  publisher={SIAM}
}

@article{rosen2019se,
  title={{SE}-{S}ync: A certifiably correct algorithm for synchronization over the special {E}uclidean group},
  author={Rosen, David M and Carlone, Luca and Bandeira, Afonso S and Leonard, John J},
  journal={The International Journal of Robotics Research},
  volume={38},
  number={2-3},
  pages={95--125},
  year={2019},
  publisher={Sage Publications Sage UK: London, England}
}

@article{cucuringu2012sensor,
  title={Sensor network localization by eigenvector synchronization over the {E}uclidean group},
  author={Cucuringu, Mihai and Lipman, Yaron and Singer, Amit},
  journal={ACM Transactions on Sensor Networks (TOSN)},
  volume={8},
  number={3},
  pages={1--42},
  year={2012},
  publisher={ACM New York, NY, USA}
}

@article{briales2017cartan,
  title={Cartan-{S}ync: Fast and global {SE}(d)-synchronization},
  author={Briales, Jesus and Gonzalez-Jimenez, Javier},
  journal={IEEE Robotics and Automation Letters},
  volume={2},
  number={4},
  pages={2127--2134},
  year={2017},
  publisher={IEEE}
}

@inproceedings{pan2024global,
  title={Global structure-from-motion revisited},
  author={Pan, Linfei and Bar{\'a}th, D{\'a}niel and Pollefeys, Marc and Sch{\"o}nberger, Johannes L},
  booktitle=ECCV,
  pages={58--77},
  year={2024},
  organization={Springer}
}

@inproceedings{sweeney2015theia,
  title={Theia: A fast and scalable structure-from-motion library},
  author={Sweeney, Christopher and Hollerer, Tobias and Turk, Matthew},
  booktitle={ACMMM},
  pages={693--696},
  year={2015}
}

@inproceedings{li2024efficient,
  title={Efficient detection of long consistent cycles and its application to distributed synchronization},
  author={Li, Shaohan and Shi, Yunpeng and Lerman, Gilad},
  booktitle={CVPR},
  pages={5260--5269},
  year={2024}
}

@inproceedings{liu20233d,
  title={3{D} line mapping revisited},
  author={Liu, Shaohui and Yu, Yifan and Pautrat, R{\'e}mi and Pollefeys, Marc and Larsson, Viktor},
  booktitle=CVPR,
  pages={21445--21455},
  year={2023}
}

@article{thirthala2012radial,
  title={Radial multi-focal tensors: Applications to omnidirectional camera calibration},
  author={Thirthala, SriRam and Pollefeys, Marc},
  journal=IJCV,
  volume={96},
  number={2},
  pages={195--211},
  year={2012},
  publisher={Springer}
}

@inproceedings{hruby2023four,
  title={Four-view geometry with unknown radial distortion},
  author={Hruby, Petr and Korotynskiy, Viktor and Duff, Timothy and Oeding, Luke and Pollefeys, Marc and Pajdla, Tomas and Larsson, Viktor},
  booktitle=CVPR,
  pages={8990--9000},
  year={2023}
}

@inproceedings{comport2007accurate,
  title={Accurate quadrifocal tracking for robust 3{D} visual odometry},
  author={Comport, Andrew I and Malis, Ezio and Rives, Patrick},
  booktitle={Proceedings 2007 IEEE International Conference on Robotics and Automation},
  pages={40--45},
  year={2007},
  organization={IEEE}
}

@inproceedings{kuang2013pose,
  title={Pose estimation with unknown focal length using points, directions and lines},
  author={Kuang, Yubin and Astrom, Kalle},
  booktitle=ICCV,
  pages={529--536},
  year={2013}
}

@inproceedings{kukelova2017clever,
  title={A clever elimination strategy for efficient minimal solvers},
  author={Kukelova, Zuzana and Kileel, Joe and Sturmfels, Bernd and Pajdla, Tomas},
  booktitle=CVPR,
  pages={4912--4921},
  year={2017}
}

@inproceedings{miraldo2018minimal,
  title={A minimal closed-form solution for multi-perspective pose estimation using points and lines},
  author={Miraldo, Pedro and Dias, Tiago and Ramalingam, Srikumar},
  booktitle={ECCV},
  pages={474--490},
  year={2018}
}

@article{kileel2018distortion,
  title={Distortion varieties},
  author={Kileel, Joe and Kukelova, Zuzana and Pajdla, Tomas and Sturmfels, Bernd},
  journal={Foundations of Computational Mathematics},
  volume={18},
  number={4},
  pages={1043--1071},
  year={2018},
  publisher={Springer}
}

@inproceedings{elqursh2011line,
  title={Line-based relative pose estimation},
  author={Elqursh, Ali and Elgammal, Ahmed},
  booktitle={CVPR},
  pages={3049--3056},
  year={2011}}

@article{kileel2017minimal,
  title={Minimal problems for the calibrated trifocal variety},
  author={Kileel, Joe},
  journal={SIAM Journal on Applied Algebra and Geometry},
  volume={1},
  number={1},
  pages={575--598},
  year={2017},
  publisher={SIAM}
}

@article{comport2010real,
  title={Real-time quadrifocal visual odometry},
  author={Comport, Andrew I and Malis, Ezio and Rives, Patrick},
  journal={The International Journal of Robotics Research},
  volume={29},
  number={2-3},
  pages={245--266},
  year={2010},
  publisher={SAGE Publications Sage UK: London, England}
}

@inproceedings{tron2015rigid,
  title={Rigid components identification and rigidity control in bearing-only localization using the graph cycle basis},
  author={Tron, Roberto and Carlone, Luca and Dellaert, Frank and Daniilidis, Kostas},
  booktitle={2015 American Control Conference (ACC)},
  pages={3911--3918},
  year={2015},
  organization={IEEE}
}

@inproceedings{arrigoni2016camera,
  title={Camera motion from group synchronization},
  author={Arrigoni, Federica and Fusiello, Andrea and Rossi, Beatrice},
  booktitle={2016 Fourth International Conference on 3D Vision (3DV)},
  pages={546--555},
  year={2016},
  organization={IEEE}
}

@book{dummit2004abstract,
  title={Abstract Algebra},
  author={Dummit, David S. and Foote, Richard M.},
  edition={3rd},
  year={2004},
  publisher={John Wiley \& Sons}}

\setcounter{page}{\value{page}} 

\clearpage
\setcounter{page}{\value{page}} 
   {
   \newpage
       \onecolumn
        \centering
        \Large
        \vspace{0.5em}Appendices \\
        \vspace{1.0em}
   }
   
\section{Additional Background Material}

\subsection{Line Projection Matrices}
Let $A\in \R^{3\times 4}$ be a projection matrix, factoring as $KR \left[I | -t\right]$  where $R\in SO(3)$ and $t\in \mathbb{R}^3$. Then, the exterior square is
\begin{align*}
    \mathcal{P} = \begin{bmatrix}
        A^2 \wedge A^3 \\
        A^3 \wedge A^1 \\
        A^1 \wedge A^2
    \end{bmatrix}
\end{align*}
where $A^i$ indicates the $i$th row of the camera matrix $A$. Here, $\wedge$ is the wedge product, or the exterior product, between two vectors. The order of the basis is given by the order in the $columnindices$ in the following explanation.  
This can be calculated specifically in the following way. 
First, let $rowindices= [2,3;~1,3;~1,2], \hspace{0.25em} columnindices = [1,2;~1,3;~1,4;~2,3;~2,4;~3,4]$. Then let $\mathcal{P}$ be the $3\times 6$ matrix, where $\mathcal{P}_{i,j}$ is calculated by the minors 
$$\det(A(rowindices(i,:), \hspace{0.25em} columnindices(j,:))$$ if $i=1,3$, and by $$-\det(A(rowindices(i,:), \hspace{0.25em} columnindices(j,:)$$ if $i=2$. When $A$ is a $3n \times 4$ matrix, then let $\mathcal{P}$ denote the $3n\times 6$ matrix of the stacked exterior squares. We refer to \cite{dummit2004abstract} for a more detailed description of the exterior algebra of a vector space. Note that the matrix $\mathcal{P}$ is also the line projection matrices, which can project 3D world lines in plucker coordinates (which are in $\R^6$) to lines in the image plane. We refer to \cite{hartley2003multiple} for more details on the role of line projection matrices in computer vision.

\section{Proofs for Theorems}
\subsection{Block Trifocal Tensor in Collinear Cases}
The reference \cite{miao2024tensor} includes a characterization of the low multilinear rank for the block trifocal tensor under the non-collinear setting. We extend this to the collinear setting, showing that the block trifocal tensor will then have a low multilinear rank of $(5,4,4)$. 
\begin{theorem}\label{thrm:tftcoll}
    Given a block trifocal tensor with appropriately chosen blockwise scales.  If the $n$ cameras that produce $\mathcal{T}^n$ are all collinear and do not all share the same camera center, then mlrank$(\mathcal{T}^n) = (5,4,4)$. 
\end{theorem}
\begin{proof}
    Since it has been shown in the proof for Theorem~\ref{thrm:tucker decomposition} that the stacked camera matrices will only exhibit a low rank if all cameras lie at the same point, it is the same case for the low ranks in the flattenings of the second and third mode. We only need to show the rank drop in the first mode. Note that $\mathcal{T}^n_{(1)} = \mathcal{P} (\mathcal{G}_T)_{(1)}(C \otimes C)^T$, we just need to show that $\mathcal{P}\in \R^{3n \times 6}$ has rank $5$. Since for individual line projection matrices $\mathcal{P}_i$, we have $\mathcal{P}_i \mathcal{L} = 0$ only when $\mathcal{L}$ is a line that passes through the camera center. Since we have a collinear set of cameras, there is only one line that passes through all of the camera centers, hence the kernel of $\mathcal{P}$ has dimension 1. 
\end{proof}

\subsection{Proof for Theorem \ref{thrm:tucker decomposition}}
\begin{proof}
    By expanding equation \eqref{eq:qfromp} and using the definition of the Tucker product
    \begin{align}\label{eq:qdecomp}
        (\mathcal{G}_Q & \times_1 C \times_2  C \times_3 C \times_4 C)_{pqrs} = \sum_{a=1}^4 \sum_{b=1}^4 \sum_{c=1}^4 \sum_{d=1}^4 (\mathcal{G}_Q)_{abcd} C_{ap} C_{bq} C_{cr} C_{ds},
    \end{align} we can see that the core $\mathcal{G}_Q$ admits the following structure: 
    \begin{align*}
    (\mathcal{G}_{Q})_{abcd} = \begin{cases}
        \text{sgn}(abcd) & \text{if } a,b,c,d \in \{1,2,3,4\} \text{ are not all the same}\\
        0 & \text{otherwise}.
    \end{cases}
    \end{align*}
 
    Then, each quadrifocal tensor block $\mathcal{Q}^n_{ijkl}$ calculated through the Tucker decomposition will be the exact same as the quadrifocal tensor calculated using \eqref{eq:qfromp}.
    Now, we establish the full rankness of $C$ in the non-collinear case. Suppose that rank$(C) < 4$, there exists $x\in \R^4$ such that $C x = 0$. This is equivalent to $P_i x = 0$ for any $i = 1,...,n$. However, $P_ix = 0$ only when $x$ is the camera centre of the camera. This means that $x$ is the camera center for all cameras, which is a contradiction. Thus, if the cameras that produce $\mathcal{Q}^n$ don't share the same camera center, the multilinear rank of $\mathcal{Q}^n = (4,4,4,4)$, even in the collinear case.
   
\end{proof}
\subsection{Proof for Theorem \ref{thrm: prank of qft}}
\begin{proof}
    We prove by explicitly calculating the contractions with two vectors. We know $\mathcal{Q}^n$ admits the factorization $\mathcal{Q}^n = \mathcal{G}_Q \times_1 C \times_2 C \times_3 C \times_4 C$, where $\mathcal{G}_Q \in \mathbb{R}^{4 \times 4 \times 4 \times 4}$. 
    Denote P-Rank($T$) as $(M_1,M_2,M_3,M_4,M_5,M_6)$. Calculating the P-Rank involves taking generic linear combinations of the slices of the tensor, which is equivalent to contracting with a matrix $X \in \R^{3n\times 3n}$. For example, for $M_1$ it can be calculated as the rank for $\mathcal{Q}^n_{p12} = \sum_{i=1}^{3n}\sum_{j=1}^{3n} X_{ij} \mathcal{Q}^n_{ij::}$. 
    \begin{itemize}
        \item[(1)] $(\mathbf{M_1})$ Recall that 
        \begin{align*}
            \mathcal{Q}^n_{ijkl} &= (\mathcal{G}_Q \times_1 C \times_2 C \times_3 C \times_4 C)_{ijkl} = \sum_{a=1}^4 \sum_{b=1}^4  \sum_{c=1}^4 \sum_{d=1}^4 \mathcal{G}_{abcd} C_{ia} C_{jb} C_{kc} C_{ld}.
        \end{align*}
        Then 
        \begin{align*}
            (\sum_{i,j}X_{ij}\mathcal{Q}^n_{ij::})_{kl} 
            =  \sum_{i,j}X_{ij} \sum_{a=1}^4 \sum_{b=1}^4  \sum_{c=1}^4 \sum_{d=1}^4 \mathcal{G}_{abcd} C_{ia} C_{jb} C_{kc} C_{ld} =  \sum_{a=1}^4 \sum_{b=1}^4  \sum_{c=1}^4 \sum_{d=1}^4 \mathcal{G}_{abcd} (\sum_{i,j}X_{ij} C_{ia} C_{jb}) C_{kc} C_{ld}.
        \end{align*}  
        Let $S_{ab} = \sum_{i,j}X_{ij} C_{ia} C_{jb}$, or in other words $S = C^TXC$. Then, we  have 
        \begin{align*}
            (\sum_{i,j}X_{ij}\mathcal{Q}^n_{ij::})_{kl} =&  \sum_{a=1}^4 \sum_{b=1}^4  \sum_{c=1}^4 \sum_{d=1}^4 \mathcal{G}_{abcd} S_{ab} C_{kc} C_{ld}  \\
            =&\sum_{c=1}^4 \sum_{d=1}^4 \left(\sum_{a=1}^4 \sum_{b=1}^4   \mathcal{G}_{abcd} S_{ab}\right) C_{kc} C_{ld} = \left(C \left(\sum_{a=1}^4 \sum_{b=1}^4   \mathcal{G}_{abcd} S_{ab}\right) C^T\right)_{kl}.
        \end{align*}
        Here, $\mathcal{G}_S \in \R^{4 \times 4}$ is given by
        \begin{align*}
            \mathcal{G}_S=&\sum_{a=1}^4 \sum_{b=1}^4   \mathcal{G}_{abcd} S_{ab} = \begin{bmatrix}
                0 & S_{34} - S_{43} & S_{42} - S_{24} & S_{23} - S_{32} \\
                S_{43} - S_{34} & 0 & S_{14} - S_{41} & S_{31} - S_{13}\\
                S_{24} - S_{42} & S_{41}-S_{14} & 0 & S_{12} - S_{21}\\
                S_{32} - S_{23} & S_{13} - S_{31} & S_{21} - S_{12} & 0
            \end{bmatrix} 
        \end{align*}
        By our definition of projection rank, $\text{rank}(X)=1$, so $X = x y^T$ for some generic $x,y \in \R^4$. In general, we can explicitly calculate the rank of $\G_S$, where the reduced row echelon form is 
        \begin{align*}
            \text{rref}(\G_S) = 
            &\begin{bmatrix}
                1 & 0 & -(x_1y_4 - x_4y_1)/(x_3y_4-x_4y_3) & -(x_1y_3 - x_3y_1)/(x_3y_4-x_4y_3) \\
                0 & 1 & -(x_2y_4 - x_4y_2)/(x_3y_4-x_4y_3) & -(x_2y_3 - x_3y_2)/(x_3y_4-x_4y_3) \\
                0 &0 & 0 & 0\\
                0 & 0 & 0 & 0
            \end{bmatrix}
        \end{align*}
        when $x_3y_4 - x_4y_3 \not = 0$. Thus, for generic $X$, $\text{rank}(\G_S) = 2$, so that $\text{rank}(C\G_SC^T) = 2$, proving the projection rank for $M_1$ to be $2$. This has also been checked numerically and symbolically. 
        
        \item[(2)] $\mathbf{M_2},~\mathbf{M_3},~\mathbf{M_4},~\mathbf{M_5},~\mathbf{M_6}$: Due to symmetries, $M_2,...,M_6$ will have the same ranks and properties. We omit the details. 
    \end{itemize}
    This establishes Theorem~\ref{thrm: prank of qft}.
\end{proof}

\subsection{Proof for Proposition \ref{prop:additional properties}}
\begin{proof}
    These are mainly facts from \cite{hartley2003multiple}, but we still describe the rationale and how it corresponds to the descriptions in the reference. Recall that the quadrifocal tensor can be calculated via the following formula \eqref{eq:qfromp} given four cameras $P_i,P_j,P_k,P_l$:
\begin{equation*}
    (\mathcal{Q}^n_{ijkl})_{pqrs} = \det 
    \begin{bmatrix}
        P_i^p \\
        P_j^q \\
        P_k^r \\
        P_l^s\\
    \end{bmatrix}.
\end{equation*}
\begin{enumerate}
    \item[(1)] For $\mathcal{Q}^n_{iiii}$ diagonal blocks, all four cameras will be the same in the definition for quadrifocal tensors from camera matrices. Since $P_i$ only has $3$ rows, there must be a repeating row in the determinant, meaning that $\mathcal{Q}^n_{iiii} = 0$ in $\mathbb{R}^{3 \times 3 \times 3 \times 3}$. 
    \item[(2)] Suppose that we now have $3$ overlapping indices, such that we are considering the blocks $\mathcal{Q}^n_{iiij},\mathcal{Q}^n_{iiji},\mathcal{Q}^n_{ijii},\mathcal{Q}^n_{jiii}$ for $i\not = j$. Then, among the four rows in the determinant, three must correspond to the same camera $i$. By \cite{hartley2003multiple}, these entries will be the epipoles up to signs, where epipoles are the images of the camera center of view $i$ in view $j$.
    \item[(3)] Suppose that we now have $2$ overlapping indices, such that we are looking at the blocks $\mathcal{Q}^n_{iijk}$, $\mathcal{Q}^n_{ijik}$, $\mathcal{Q}^n_{ijki}$, $\mathcal{Q}^n_{jiik}$, $\mathcal{Q}^n_{jiki}$, $\mathcal{Q}^n_{jkii}$, $\mathcal{Q}^n_{iikj}$, $\mathcal{Q}^n_{ikij}$, $\mathcal{Q}^n_{ikji}$, $\mathcal{Q}^n_{kiij}$, $\mathcal{Q}^n_{kiji}$, $\mathcal{Q}^n_{kjii}$ for $i\not = j, i\not = k, j\not = k$. Then, among the four rows in the determinant, two correspond to the same camera $i$. Again by \cite{hartley2003multiple}, these entries will be the elements of the trifocal tensor up to signs. 
    \item[(4)]Suppose that we now have $2$ overlapping indices, and the other two are also the same, such that we are looking at the blocks, $\mathcal{Q}^n_{iijj},\mathcal{Q}^n_{ijij},\mathcal{Q}^n_{ijji},\mathcal{Q}^n_{jiij},\mathcal{Q}^n_{jiji},\mathcal{Q}^n_{jjii}$ for $i\not = j$.  
    Then, among the four rows in the determinant, two correspond to the same camera $i$, and the other two correspond to the same camera $j$. Again by \cite{hartley2003multiple}, these entries will be the elements of the fundamental matrices up to signs. 
\end{enumerate}
This completes the proof.
\end{proof}

\subsection{Proof for Theorem \ref{thrm:quadscales}}
\begin{proof}
Blockwise multiplication by a rank-1 tensor with non-vanishing entries preserves multilinear rank, as this operation is equivalent to a Tucker product with invertible diagonal matrices.

Hence, without loss of generality, assume $\lambda_{i111} = \lambda_{1j11} = \lambda_{11k1} = \lambda_{111\ell} = 1$ for all $i,j,k, \ell \in \{2, \ldots, n\}$. We will show below that for some $c \in \mathbb{R}^*$, the entries satisfy:
\begin{itemize}
    \item $\lambda_{ijk\ell} = c$ if exactly two of $i,j,k,\ell$ equal $1$;
    \item $\lambda_{ijk\ell} = c^2$ if exactly one of $i,j,k$ equals $1$;
    \item $\lambda_{ijk\ell} = c^3$ if none of $i,j,k,\ell$ equal $1$ and the indices are not all the same.
\end{itemize}
This result will establish the theorem, since setting $\alpha = \beta = \gamma = (1, c, \ldots, c)$ and $\delta = (\tfrac{1}{c}, 1, \ldots, 1)$ ensures $\lambda_{ijk\ell} = \alpha_i \beta_j \gamma_k \delta_{\ell}$ whenever $i,j,k, \ell$ are not all equal.

We let $\mathcal{Q}^n_{(1)}$ and $(\lambda \odot_b \mathcal{Q}^n)_{(1)}$ denote the mode-$1$ matrix flattenings in $\mathbb{R}^{3n \times 27n^3}$ of the block quadrifocal tensor and its scaled variant, where rows correspond to the first mode of the tensors. Invoking Theorem~\ref{thrm:tucker decomposition} alongside our standing assumptions shows that both matrices have rank $4$. Consequently, every $5 \times 5$ minor of these matrices must be zero.

We will exploit this by analyzing specific $5 \times 5$ submatrices of $(\lambda \odot_b \mathcal{Q}^n)_{(1)}$ to construct a system of constraints on $\lambda$, ultimately establishing the existence of the constant $c$.
We adopt the index convention $(ip)$ for rows and $(jq, kr, \ell s)$ for columns, where $i,j,k,\ell \in [n]$ and $p,q,r,s \in [3]$. Thus, entries satisfy $((\lambda \odot_b \mathcal{Q}^n)_{(1)})_{(ip), (jq, kr, \ell s)} = \lambda_{ijk\ell} (\mathcal{Q}^n_{ijk \ell})_{pqrs}$.

\textbf{Case 1:} The first submatrix of $(\lambda \odot_b \mathcal{Q}^n)_{(1)}$ we consider has column indices
$(i1, 13, 12)$, 
$(12, j2, 11)$,
$(12, j3, 11)$,
$(13, j3, 12)$,
$(11, j1, 13)$
and row indices 
$(11)$,
$(12)$,
$(13)$,
$(i1)$,
$(i2)$, 
where $i, j \in \{2, \ldots, n\}$.
Explicitly, the submatrix \nolinebreak is  
\begin{equation*}
\begin{bmatrix}
(\mathcal{Q}^n_{1i11})_{1132} & (\mathcal{Q}^n_{11j1})_{1221} & (\mathcal{Q}^n_{11j1})_{1231} & (\mathcal{Q}^n_{11j1})_{1332} & (\mathcal{Q}^n_{11j1})_{1113} \\[5pt] 
(\mathcal{Q}^n_{1i11})_{2132} & (\mathcal{Q}^n_{11j1})_{2221} & (\mathcal{Q}^n_{11j1})_{2231} & (\mathcal{Q}^n_{11j1})_{2332} & (\mathcal{Q}^n_{11j1})_{2113} \\[5pt]
(\mathcal{Q}^n_{1i11})_{3132} & (\mathcal{Q}^n_{11j1})_{3221} & (\mathcal{Q}^n_{11j1})_{3231} & (\mathcal{Q}^n_{11j1})_{3332} & (\mathcal{Q}^n_{11j1})_{3113} \\[4pt]
\lambda_{ii11}(\mathcal{Q}^n_{ii11})_{1132} & \lambda_{i1j1}(\mathcal{Q}^n_{i1j1})_{1221} & \lambda_{i1j1}(\mathcal{Q}^n_{i1j1})_{1231} & \lambda_{i1j1}(\mathcal{Q}^n_{i1j1})_{1332} & \lambda_{i1j1}(\mathcal{Q}^n_{i1j1})_{1113} \\[4pt]
\lambda_{ii11}(\mathcal{Q}^n_{ii11})_{2132} & \lambda_{i1j1}(\mathcal{Q}^n_{i1j1})_{2221} & \lambda_{i1j1}(\mathcal{Q}^n_{i1j1})_{2231} & \lambda_{i1j1}(\mathcal{Q}^n_{i1j1})_{2332} & \lambda_{i1j1}(\mathcal{Q}^n_{i1j1})_{2113}
\end{bmatrix},
\end{equation*}
which we abbreviate as 
\begin{equation} \label{eq:my-det1}
\begin{bmatrix}
\ast & \ast & \ast & \ast & \ast \\
\ast & \ast & \ast & \ast & \ast \\
\ast & \ast & \ast & \ast & \ast \\
\lambda_{ii11}\ast & \lambda_{i1j1}\ast & \lambda_{i1j1}\ast & \lambda_{i1j1}\ast & \lambda_{i1j1}\ast \\
\lambda_{ii11}\ast & \lambda_{i1j1}\ast & \lambda_{i1j1}\ast & \lambda_{i1j1}\ast & \lambda_{i1j1}\ast \\
\end{bmatrix},
\end{equation}
with asterisk denoting the corresponding entry in $\mathcal{Q}^n_{(1)}$. 
We view the determinant of \eqref{eq:my-det1} as a polynomial with respect to $\lambda$.
It has degree $\leq 2$ in the variables $\lambda_{ii11}, \lambda_{i1j1}$. 
Observe that if $\lambda_{i1j1} = 0$, the bottom two rows of the matrix are linearly independent.  
Also if $\lambda_{i1j1} - \lambda_{ii11} = 0$, then 
\eqref{eq:my-det1} equals a $5 \times 5$ submatrix of $\mathcal{Q}^n_{(1)}$ with rows operations performed; therefore \eqref{eq:my-det1} is rank-deficient.
It follows that the determinant of \eqref{eq:my-det1} takes the form
\begin{equation*}
s \lambda_{i1j1}(\lambda_{i1j1} - \lambda_{ii11}).
\end{equation*}
Here the scale $s = s(P_1, P_i, P_j)$ is a polynomial in the camera matrices. 
Due to polynomiality, $s$ is nonzero Zariski-generically if we can exhibit a \textit{single} instance of matrices $P_1, P_i, P_j$ where the determinant of \eqref{eq:my-det1} does not vanish identically for all $\lambda_{i1j1}, \lambda_{ii11}$.
Furthermore, we just need an instance with $i = j$, as this corresponds to a specialization of the case $i \neq j$.  
Computational verification with a random numerical instance of $P_1, P_i$ proves the non-vanishing.  
Recalling the standing assumptions, we deduce 
$\lambda_{i1j1} = \lambda_{ii11}$.

We apply the same argument to modewise permutations of $\lambda \odot_b \mathcal{Q}^n$ and $\mathcal{Q}^n$, and obtain
\begin{equation*}\label{eq:nice-1}
\lambda_{\pi(i1j1)} = \lambda_{\pi(ii11)} \quad \text{for all } i, j \in \{2, \dots, n\} \text{ and permutations } \pi.
\end{equation*}
The argument goes through as $\pi \cdot \mathcal{Q}^n$ and $\pi \cdot (\lambda \odot_b \mathcal{Q}^n)$ have multilinear ranks bounded by $(4,4,4,4)$ and from equation \ref{eq:symmetry}.  
So \eqref{eq:my-det1} looks the same but with indices permuted and possibly a sign flip.

We now see that $\lambda$-entries with  two $1$-indices agree.
Indeed, taking $i = j$ above gives $\lambda_{\pi_1(i1i1)} = \lambda_{\pi_2(ii11)}$ for all $\pi_1$ and $\pi_2$ that fix $(ii11)$ and $(i1i1)$ respectively.  
So $\lambda_{ii11} = \lambda_{\pi(ii11)}$ for all $\pi$.  
Taking $i \neq j$ gives $\lambda_{ii11} = \lambda_{\pi(i1j1)} = \lambda_{jj11}$ for all $\pi$.  
Together, there exists $c \in \mathbb{R}^*$ such that 
$c = \lambda_{\pi(ij11)}$ for all $i, j \in \{2, \dots, n\}$ and permutations \nolinebreak $\pi$.

\textbf{Case 2:}  
Next we consider the submatrix of $(\lambda \odot_b \mathcal{Q}^n)_{(1)}$ with column indices  
$(j1, k3, 12)$, 
$(12, j2, 11)$,
$(12, j3, 11)$,
$(13, j3, 12)$,
$(11, j1, 13)$
and row indices 
$(11)$,
$(12)$,
$(13)$,
$(i1)$,
$(i2)$, 
where $i, j, k \in \{2, \ldots, n\}$.
It looks like 
\begin{equation} \label{eq:my-det2}
\begin{bmatrix}
c\ast & \ast & \ast & \ast & \ast \\
c\ast & \ast & \ast & \ast & \ast \\
c\ast & \ast & \ast & \ast & \ast \\
\lambda_{i j k 1}\ast & c\ast & c\ast & c\ast & c\ast \\
\lambda_{i j k 1}\ast & c\ast & c\ast & c\ast & c\ast \\
\end{bmatrix},
\end{equation}
where asterisks denote corresponding entries in $\mathcal{Q}^n_{(1)}$.
As a polynomial in $c$ and $\lambda_{i j k 1}$, the determinant of \eqref{eq:my-det2} is a scalar multiple of $c(c^2 - \lambda_{i j k 1})$.
This is because the polynomial has degree $\leq 3$, if $c=0$ then the bottom two rows of \eqref{eq:my-det2} are linearly dependent, and if $c^2 = \lambda_{i j k 1}$ then \eqref{eq:my-det2} is a $5 \times 5$ submatrix of $\mathcal{Q}^n_{(1)}$ with row and column operations performed.  
The scale is a polynomial in $P_1, P_i, P_j, P_k$.
It is Zariski-generically nonzero if we exhibit one instance of camera matrices such that the determinant of \eqref{eq:my-det1} does not vanish for all $c, \lambda_{i j k 1}$.
Further, it suffices to find an instance where $i = j = k$, as all other cases specialize to this.  Computational verification with a random numerical instance of $P_1, P_i$ proves the non-vanishing.  
It follows that $c^2 = \lambda_{i j k 1}$. 
Appealing to symmetry like before, 
$
c^2 = \lambda_{\pi(i j k 1)}
$
for all $i, j, k \in \{2, \dots, n\}$ and permutations $\pi$.  
Summarizing, all $\lambda$-entries with a single $1$-index equal $c^2$.

\textbf{Case 3:}
Consider  the submatrix of $(\lambda \odot \mathcal{Q}^n)_{(1)}$ with columns 
$(j1, k3, \ell 2)$, 
$(12, i2, 11)$,
$(12, i3, 11)$,
$(13, i3, 12)$,
$(11, i1, 13)$
and rows 
$(11)$,
$(12)$,
$(13)$,
$(i1)$,
$(i2)$, 
where $i, j, k, \ell \in \{2, \ldots, n\}$ and $i, \ell$ are distinct.
The submatrix looks like 
\begin{equation} \label{eq:my-det3}
\begin{bmatrix}
c^2\ast & \ast & \ast & \ast & \ast \\
c^2\ast & \ast & \ast & \ast & \ast \\
c^2\ast & \ast & \ast & \ast & \ast \\
\lambda_{i j k \ell}\ast & c\ast & c\ast & c\ast & c\ast \\
\lambda_{i j k \ell}\ast & c\ast & c\ast & c\ast & c\ast \\
\end{bmatrix}.
\end{equation}
The determinant of \eqref{eq:my-det3} is $c(c^3 - \lambda_{i j k \ell})$ multiplied by a polynomial in $P_1, P_i, P_j, P_k, P_{\ell}$.
The most specialized case is $i=j=k$.  Computer verification with a random numerical instance proves the polynomial is not identically zero.
We deduce that $c^3 = \lambda_{i j k \ell}$.  
By symmetry, $c^3 = \lambda_{\pi(i j k \ell)}$ for all $i, j, k, \ell \in \{2, \dots, n\}$ with $i, \ell$ distinct and all permutations $\pi$.
In other words, $\lambda$-entries with no $1$-indices and non-identical indices equal $c^3$.

Putting cases 1, 2 and 3 together shows that  $\lambda$ takes the announced form.  This proves Theorem~\ref{thrm:quadscales}.
\end{proof}

\subsection{Proof for Theorem \ref{thrm:ess decomposition}}
\begin{proof} The proof of this theorem follows from explicit calculation. We just need to verify that  \begin{equation*}
        \mathcal{E}^n_{ij} = \mathcal{P}_i \begin{pmatrix}
            0 & 0 & 0 & 0 & 0 & 1\\
            0 & 0 & 0 & 0 & 1 & 0\\
            0 & 0 & 0 & -1 & 0 & 0\\
            0 & 0 & -1 & 0 & 0 & 0\\
            0 & 1 & 0 & 0 & 0 & 0\\
            -1 & 0 & 0 & 0 & 0 & 0
        \end{pmatrix} \mathcal{P}_j^T.
\end{equation*}

For each entry $(\mathcal{E}^n_{ij})_{kl}$ in $\mathcal{E}^n_{ij}$, 
\begin{align*}
    (\mathcal{E}^n_{ij})_{kl} &= (-1)^{k+l} \det \begin{bmatrix}
        \sim P_i^l \\
        \sim P_j^k
    \end{bmatrix} = (\mathcal{P}_i)_k \begin{pmatrix}
            0 & 0 & 0 & 0 & 0 & 1\\
            0 & 0 & 0 & 0 & 1 & 0\\
            0 & 0 & 0 & -1 & 0 & 0\\
            0 & 0 & -1 & 0 & 0 & 0\\
            0 & 1 & 0 & 0 & 0 & 0\\
            -1 & 0 & 0 & 0 & 0 & 0
        \end{pmatrix} (\mathcal{P}_j)_l^T
\end{align*}
by expanding every entry. A direct expansion shows the result. This has also been checked numerically.
\end{proof}
\subsection{Proof for Theorem \ref{thrm:tft p rank}}
\begin{proof}
    We prove this by explicit calculation. The block tensor $\mathcal{T}^n$ admits the factorization $\mathcal{T}^n = \mathcal{G} \times_1 \mathcal{P} \times_2 C \times_3 C$, where $\mathcal{G} \in \mathbb{R}^{6\times 4 \times 4}$. Denote P-Rank($\mathcal{T}^n$) as $(M_1, M_2, M_3)$. Calculating the P-Rank involves taking genericlinear combinations of the slices of the tensor, this is equivalent to contracting with a vector. In the tensor-matrix product, let $x\in \mathbb{R}^{3n\times 1}$, $\sum_{i=1}^{3n} x_i \mathcal{T}^n_{ijk} = \mathcal{G} \times_1 (x^T P) \times_2 C \times_3 C$. Then we explicitly calculate $M_1$ from $\mathcal{G} \times_1 (x^T P)$. The same process holds for modes 2 and 3 to compute $M_2$ and $M_3$ in the P-Rank. 
    \begin{itemize}
        \item[(1)]$\mathbf{M_1}$: Let $y = x^T P $. We have 
        $$\mathcal{T}^n_1 = \sum_{i=1}^4 y_i \mathcal{G}_{ijk} =         \begin{bmatrix}
            0  & y_6 & -y_5 & y_4 \\
            -y_6 & 0 & y_3 & -y_2 \\
            -y_5 & -y_3 & 0 & y_1 \\
            -y_4 & y_2 & -y_1 & 0
        \end{bmatrix}.$$
        Generically, this matrix $\mathcal{T}^n_1$ will clearly have rank $4$. 
        \item[(2)]$\mathbf{M_2}$: Let $y = x^T C$, then we have 
        \begin{align*}
            \mathcal{T}^n_2 = \sum_{j=1}^4 y_j \mathcal{G}_{ijk} =  \begin{bmatrix}
            0 &0&0&y_4&-y_3&y_2\\
            0 & -y_4 & y_3 & 0 & 0 & -y_1 \\
            y_4 & 0 & -y_2 & 0 & y_1 & 0 \\
            -y_3 & y_2 & 0 & -y_1 & 0 & 0 
        \end{bmatrix}
        \end{align*}
        This matrix generically will clearly have rank $3$. Denote $R_i$ as the $i$-th row of $\mathcal{T}^n_2$. We can observe that the fourth row is a linear combination of rows 1,2,3, where $y_4 R_4 = y_1 R_1 + y_2 R_2 + y_3 R_3$. However, $R_1,R_2,R_3$ must be linearly independent in general. 
        \item[(3)] $\mathbf{M_3}$: This is the same as $M_2$ due to symmetry. 
    \end{itemize}
\end{proof}

\section{Algorithm Details}
\subsection{QuadSync}
We present pseudocode for QuadSync in Algorithm \ref{alg:quadsync} in this section. 
\begin{algorithm}[htp!]
    \caption{Quadsync IRLS-ADMM}\label{alg:quadsync}
    \begin{algorithmic}
        \State \textbf{Input: } $\mathcal{Q}^n\in \R^{3n\times 3n \times 3n \times 3n}$, $\rho>0 \in \R$
        \State \quad \quad \quad $\Omega$ set of observed block indices 
        \State \textbf{Output: } $\bar{C} \in \R^{3n\times 4}$
        \State Normalize $\mathcal{Q}^n$ so that each block has norm 1  
        \State Obtain $C_i$, $B$ from first four singular vectors in first factor matrix of HOSVD$(\mathcal{Q}^n)$
        \State Calculate initial IRLS weights $W_Q$ from \eqref{weights}
        \While{not converged}
        \State \textit{\# Now we run the ADMM cases}
            \While{not converged}
                \State \textit{\# First optimize for $C_i,\Lambda$ alternatingly}
                \While{not converged}
                    \State Update $C_i$, $i = 1, 2, 3, 4$ via \eqref{Cupdate}
                    \State Update $\Lambda_Q$ via \eqref{Lupdate}
                \EndWhile
                \State Update $B$ via \eqref{Bupdate}
                \State Update $\Gamma_i$, $i=1, 2, 3, 4$ via \eqref{Gupdate}
            \EndWhile
            \State Update new IRLS weights $W_Q$ via \eqref{weights}
        \EndWhile
        \State $\bar{C} = \text{avg}(C_1, C_2, C_3, C_4)$
        \State \textbf{Return}
    \end{algorithmic}
\end{algorithm}

We count the storage and flop complexities of QuadSync. Let $n$ be the number of cameras. Then the input of QuadSync consists of a block quadrifocal tensor with up to $81n^4$ elements. 
The auxiliary variables consist of $C_1,~C_2,~C_3,~C_4,~\Lambda,~W,~B,~\Gamma_1,~\Gamma_2,~\Gamma_3$, with size $12n \times 8 + 2|\Omega| = O(n^4)$. Thus, the overall storage complexity is $O(n^4)$.

For the flop complexity, we first account for the initialization. For the initialization, $C$ is calculated via HOSVD of a matrix of size $3n \times 27n^3$ where the target rank is $4$. This costs costs $O(4 \times 81 \times n^4) = O(n^4)$ flops. For the running  iterations, to calculate $C$, one has to multiply a matrix of size $4\times 27n^3$ with a diagonal matrix of size $27n^3 \times 27n^3$, then take the product and multiply it with another matrix of size $27n^3 \times 4$. This multiplication of three matrix will have to be done for a total of $n$ times for each update of the stacked camera matrices $C$. Thus, the flop complexity for calculating the updates for $C$ is $O(n^4)$. Calculating $\Lambda_Q,~W$ also requires $O(|\Omega|) = O(n^4)$ flops. Calculating $B,~\Gamma_i$ requires $O(n)$ flops. The overall flop complexity is $O(n^4)$. 

\subsection{Joint Opt.}
Though similar to the above, we include  detailed calculations for the steps of Joint Opt. and include  pseudocode in Algorithm~\ref{alg:jointopt}. We first give explicit formulas for all the updates that we use. 

\begin{itemize}
\item[(1)] $\mathbf{W_Q}, ~\mathbf{W_T},~\mathbf{W_E}:$
Here $W$ is the set of IRLS weights integrated in the norm, so
\begin{align*} w_{ijkl} = \begin{cases}
    1 / m^t_{ijkl} \quad \text{ if } (i,j,k,l)\in \Omega\\
    0 \quad \quad \quad \quad \quad \text{otherwise}.
\end{cases} \end{align*}
For $W_Q$ we have 
\begin{align}\label{jointWeightsQ}
    (m_Q)^t_{ijkl} = \max(\delta,~sqrt(\|(\Lambda_Q)^{(t-1)}_{ijkl} \tilde{\mathcal{Q}}^n_{ijkl} - \G_Q \times_1 (C_1)^{(t-1)}_i  \times_2 (C_2)^{(t-1)}_j \times_3 (C_3)^{(t-1)}_k \times_4 (C_4)^{(t-1)}_l)\|_F).
\end{align}

For $W_T$ we have 
\begin{align}\label{jointWeightsT}
    (m_T)^t_{ijkl} = \max(\delta,~sqrt(\|(\Lambda^{(t-1)}_T)_{ijkl} \tilde{\mathcal{T}}^n_{ijkl} - \G_T \times_1 (P_1)^{(t-1)}_i \times_2 (C_5)^{(t-1)}_j \times_3 (C_6)^{(t-1)}_k \|_F).
\end{align}

For $W_E$ we have
\begin{align}\label{jointWeightsE}
    (m_E)^t_{ijkl} = \max(\delta,~sqrt(\|(\Lambda_E)^{(t-1)}_{ijkl} \tilde{\mathcal{E}}^n_{ijkl} - \G_E  \times_1 (P_2)^{(t-1)}_i   \times_2 (P_3)^{(t-1)}_j \|_F).
\end{align}

\item[(2)]$\mathbf{C_1},~\mathbf{C_2},~\mathbf{C_3},~\mathbf{C_4}:$ For the $i$th mode, let $C_{\sim i} = C_4 \otimes \cdots \otimes C_{i+1 }\otimes C_{i-1} \otimes \cdots \otimes C_1$, and let $K = (\G_Q)_{(i)} C_{\sim i}^T $. Then, let $x_j$ denote the $i$th row of the variable being optimized for. Then
    \begin{align} \label{jointC1update}
        x_j  = (\frac{\rho}{2}(B-\Gamma_i)_j + (\frac{1}{n_Q}(W^2_{(i)})_j \odot_b [(\Lambda \odot_b \tilde{\mathcal{Q}}^n)_{(i)}]_{j}) K^T)( \frac{\rho}{2}I_{4\times 4} + \frac{1}{n_Q} K \text{diag}((W^2_{(i)})_j) K^T)^{-1}.
    \end{align}
\item[(3)]$\mathbf{C_5},~\mathbf{C_6},~\mathbf{P_1}:$ To solve for $C_i, ~i=5,6$, let $C_{\sim i} = P \otimes C_j$, and let $K = (\G_T)_{(i)} C_{\sim i}^T $. Then, we have the update rule for $C_i$ as 
\begin{align}\label{jointC5update}
    x_j = (\frac{\rho}{2}( B -\Gamma_i)_j + \frac{1}{n_T}((W_T^2 \odot L_T) \odot_b T)_j K^T)(\frac{\rho}{2}I_{4\times 4} + \frac{1}{n_T} K \text{diag}((W_T)_j^2)K^T)^{-1}.
\end{align}
For $P_1$, let $K = (\G_T)_{(3)}(C_2 \otimes C_1)^T$. Then
\begin{align}\label{jointP1update}
    x_j = (\frac{\rho}{2}( D -\tau_i)_j + \frac{1}{n_T}((W_T^2 \odot L_T) \odot_b T)_j K^T)(\frac{\rho}{2}I_{6\times 6} + \frac{1}{n_T} K \text{diag}(W_T^2)K^T)^{-1}.
\end{align}

\item[(4)]$\mathbf{P_2},~\mathbf{P_3}:$ To solve for $P_i,~i=2,3$, let $K = (\G_E)_{(i)} P_{j}^T $ where $j\not = i$. Then, we have the update rule for $P_i$ as  \begin{align}\label{jointP2update}
    x_j = (\frac{\rho}{2}( D -\tau_i)_j + \frac{1}{n_T}((W_E^2 \odot L_E) \odot_b E)_j K^T)(\frac{\rho}{2}I_{6\times 6} + \frac{1}{n_E} K \text{diag}(W_E^2)K^T)^{-1}.
\end{align}

\item[(5)]$\mathbf{\Lambda_Q},~\mathbf{\Lambda_T},~\mathbf{\Lambda_E}:$ For $\Lambda$, we can solve as usual by directly solving the convex optimization problem associated with $f_Q,~g_T,~h_E$,
    This is done by solving for each block separately, where 
    \begin{align}\label{jointLqupdate}
    (\Lambda_Q)_{ijkl} = \frac{trace((\llbracket \mathcal{G}_Q;~C_1,~C_2,~C_3,~C_4 \rrbracket_{ijkl})_{(1)}^T ((\tilde{\mathcal{Q}}^n)_{ijkl}))_{(1)})}{\|((\tilde{\mathcal{Q}}^n)_{ijkl})_{(1)}\|_F^2},
    \end{align}

\begin{align}\label{jointLtupdate}
    (\Lambda_T)_{ijk} = \frac{trace((\llbracket \mathcal{G}_T;~P_1,~C_5,~C_6 \rrbracket_{ijk})_{(1)}^T ((\tilde{\mathcal{T}}^n)_{ijk}))_{(1)})}{\|((\tilde{\mathcal{T}}^n)_{ijk})_{(1)}\|_F^2},
    \end{align}

\begin{align}\label{jointLeupdate}
    (\Lambda_E)_{ij} = \frac{trace((\llbracket \mathcal{G}_E;~P_2,~P_3\rrbracket_{ij})_{(1)}^T ((\tilde{\mathcal{E}}^n)_{ij}))_{(1)})}{\|((\tilde{\mathcal{E}}^n)_{ij})_{(1)}\|_F^2}.
    \end{align}
    
After calculating $\Lambda_Q,~\Lambda_T,~\Lambda_E$, for each of them, we first symmetrize, then normalize $\Lambda$ so that $\|\Lambda\|_F^2 = 1$ and $\Lambda$ satisfies the required symmetry requirements.
 
\item[(6)]$\mathbf{B},\mathbf{D}:$ For $B$, we can solve directly, where 
    \begin{equation}\label{jointBupdate}
        B = \frac{1}{6} \left( \sum_{i=1}^6 C_i + \Gamma_i \right).
    \end{equation}
    For $D$, we also solve directly, where
     \begin{equation}\label{jointDupdate}
        D = \frac{1}{3} \left( \sum_{i=1}^3 P_i + \tau_i \right).
    \end{equation}
    
\item[(7)]$\mathbf{\Gamma_i},~\mathbf{\tau_i}:$ For the final ascent cases for $\Gamma_i,~\tau_i$, we set \begin{equation}\label{jointGupdate}
    \Gamma_{i} = \Gamma_i + (C_i - B),
    \end{equation}
    \begin{equation}\label{jointTauupdate}
    \tau_{i} = \tau_i^{(k)} + (P_i - D).
    \end{equation}
\end{itemize}

\begin{algorithm}[htp!]
    \caption{Joint Opt. IRLS-ADMM}\label{alg:jointopt}
    \begin{algorithmic}
        \State \textbf{Input: } $\mathcal{Q}^n\in \R^{3n\times 3n \times 3n \times 3n}$, $\mathcal{T}^n\in \R^{3n\times 3n \times 3n}$, $\rho>0$ in $\R$
        \State \quad \quad \quad $\mathcal{E}^n\in \R^{3n\times 3n \times 3n}$, $\Omega$ observed indices 
        \State \textbf{Output: } $\bar{C} \in \R^{3n\times 4}$
        \State Normalize $\mathcal{Q}^n,~ \mathcal{T}^n,~\mathcal{E}^n$ so that each block has norm 1  
        \State Calculate $n_Q,~n_T,~n_E$ as the number of estimated blocks respectively. 
        \State Obtain $C_i,~B$ from first four singular vectors in first factor matrix of  HOSVD$(\mathcal{Q}^n)$
        \State Calculate initial IRLS weights $W_Q,~W_T,~W_Q$ from \eqref{jointWeightsQ}, \eqref{jointWeightsT}, \eqref{jointWeightsE}. 
        \While{not converged}
        \State \textit{\# Now we run the ADMM steps}
            \While{not converged}
                \State \textit{\# First optimize for $C_i,~P_i,~\Lambda
                _Q,~\Lambda_T,~\Lambda_E$ alternatingly}
                \While{not converged}
                    \State Update all $C_i$, $P_j$,  for $j=1,2,3$ via \eqref{jointC1update}, \eqref{jointC5update}, \eqref{jointP1update},  \eqref{jointP2update}
                    \State Update $\Lambda_Q,~\Lambda_T,~\Lambda_E$ via \eqref{jointLqupdate}, \eqref{jointLtupdate}, \eqref{jointLeupdate}
                \EndWhile
                \State Update $B$ via \eqref{jointBupdate}
                \State Update $D$ via \eqref{jointDupdate}
                \State Update $\Gamma_i$ via \eqref{jointGupdate}
                \State Update $\tau_i$ via \eqref{jointTauupdate}
            \EndWhile
            \State Update new IRLS weights $W_Q,~W_T,~W_E$ via \eqref{jointWeightsQ}, \eqref{jointWeightsT}, \eqref{jointWeightsE}
        \EndWhile
        \State $\bar{C} = \text{avg}(C_1,~C_2,~C_3,~C_4)$
        \State \textbf{Return}
    \end{algorithmic}
\end{algorithm}

We note that in the joint optimization, we only constrain the equality between $C_i$'s and $P_i$'s. However, there is also a relationship between $C_i$ and $P_i$, where $P_i$'s elements should be the $2\times 2$ minors of $C_i$. In future work, this constraint could be further explored to potentially improve the algorithm and the linkage between the factor matrices and the three block entities. 
\section{Additional Experiments}
\subsection{Randomized Updates in QuadSync}
We test the effect of randomized updates for $C_i$ in QuadSync for the ETH3D `relief' dataset. Our `relief' dataset contains 13 images, and the maximum number of columns for updating $C_i$ is $27 \times 13^3 = 59319$. We try randomized updates using columns whose number range from 20 to 59319. Using 30 random columns achieve similar accuracy with notable speed-ups. See Figure~\ref{fig:rebuttal_speed} for the effect of the number of columns on accuracy and runtime.  
\begin{figure}[htp!]
    \centering
    \includegraphics[width=0.8\linewidth]{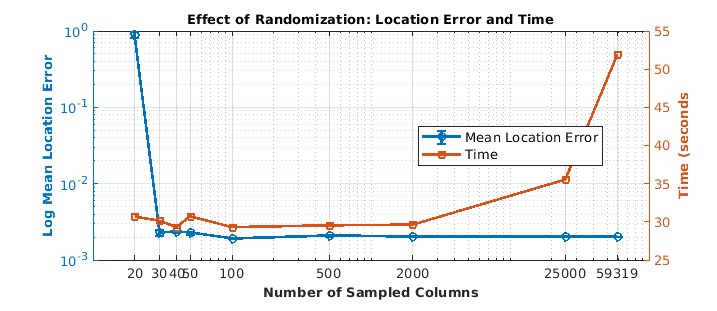}
    \caption{Randomized updates in QuadSync tested on ETH3D `relief' dataset.}
    \label{fig:rebuttal_speed}
\end{figure}

\subsection{Collinear Experiments}
\begin{figure}[htpb!]
    \centering
    \includegraphics[width=0.7\linewidth]{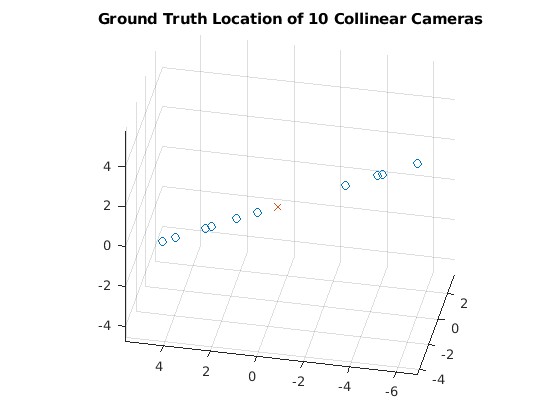}
    \caption{Ground truth location of 10 cameras.  {\color{red} $\times$}} {\begin{small} is the origin.\end{small}} {\color{blue} $\circ$} {\begin{small} denotes the camera centers. \end{small}}
    \label{fig:collinear demonstration}
\end{figure}

Given a set of images, it is well known that the locations can only be uniquely determined up to a global transformation, only when the graph is \textit{parallel rigid}, see \cite{tron2015rigid} and \cite{arrigoni2016camera} for more details on parallel rigidity in computer vision. The most degenerate case is when all of the cameras are collinear, meaning that they lie on a common line. In this case, synchronization algorithms that involve solving the location synchronization problem using solely pairwise directions, will fail. However, collinear or close to collinear situations are common in real life, such as in self-driving cars or robot motion. 
We conduct a small set of synthetic experiments to demonstrate that our algorithm can successfully recover camera poses in the collinear setting. A simple diagram can show the situation, see Figure~\ref{fig:collinear demonstration}.

We randomly generate calibrated collinear cameras. Then, for each individual quadrifocal tensor estimation, we add random noise of a certain percentage to its corresponding cameras and form the corresponding block quadrifocal tensor individually for each quadruplet. The noise parameter corresponds to the amount of noise we add to the cameras. We then randomly sample a certain percentage of quadruplets. We discard all the trifocal tensor and fundamental matrices. The results are then aggregated in Table~\ref{tab:collinear experiments}.
\begin{table}[htpb!]
    \centering
    \begin{tabular}{cccccc}
         \hline
         perc. (\%) & noise(\%) & $\bar{e_t}$ & $\hat{e_t}$ & $\bar{e_r}$ & $\hat{e_r}$\\
         \hline
         100 & 0.00 & 0.00 & 0.00 & 0.00 & 0.00 \\
         100 & 1 & 0.04 & 0.04 & 0.23 & 0.18 \\
         100 & 5 & 0.24 & 0.22 & 2.67 & 2.52 \\
         80 & 0 & 0.00 & 0.00 & 0.00 & 0.00\\
         80 & 1 & 0.04 & 0.03 & 0.37 & 0.37 \\
         80 & 5 & 0.36 & 0.36 & 3.07 & 2.68 \\
         60 & 0 & 0.00 & 0.00 & 0.00 & 0.00\\
         60 & 1 & 0.06 & 0.05 & 0.42 & 0.38 \\
         60 & 5 & 0.64 & 0.54 & 4.52 & 3.67 \\
         \hline
    \end{tabular}
    \caption{Results for synthetic  experiments with collinear cameras. $\bar{e_t}$: mean location error. $\hat{e_t}$: median location error. $\bar{e_r}$: mean rotation error. $\hat{e_r}$: median rotation error. }
    \label{tab:collinear experiments}
\end{table}
Our algorithm can still recover the camera parameters effectively in the collinear setting. Standard pipelines that synchronization rotation and location separately will fail in this situation. Global synchronization algorithms using the block essential matrix and the block trifocal tensors will need special tuning, as the ranks of the entities drop. The block quadrifocal tensor is insensitive to such cases and has a special advantage in the collinear motion setting. 

\subsection{Distributed Synchronization}

Our quadrifocal tensor synchronization algorithm becomes slow as it deals with higher order tensors. However, it can be sped up with distributed synchronization, where clusters are formed and solved in parallel. In our set of experiments, we follow the way we add noise to the cameras in the collinear experiments. We artificially handpick the clusters, so that each cluster is fully dense (we continue to disregard trifocal and fundamental matrix elements). There are overlapping cameras in each cluster, and we finally synchronize all the clusters by calculating the projective transformation that maps the overlapping cameras to be in the same projective frame. The main take away for this experiment is to demonstrate the potential for applicability of the algorithm on large datasets. We conduct sets of experiments with just 2 CPU cores, and show that by focusing on smaller subsets of the data and using parallelization, we can greatly reduce the runtime of our algorithm, making it more scalable and applicable to larger scenes. 

\textbf{Experiment:} We compare just the difference in runtime for a small synthetic dataset. We add no noise and assume that the entire block is observed accurately. We break the entire set of 30 cameras into 3 clusters. The second and first cluster have 5 overlapping cameras, same with the second and third. We then synchronize each cluster separately. The resulting cameras will be in different projective frames. We then use the overlapping cameras to align the clusters to the same projective frame. Noise is added similarly as in the synthetic collinear experiments. We include the results in Table~\ref{tab:distributed synchronization}. 

We also include a small experiment with handpicked clusters for the EPFL CastleP30 dataset. Note that we can not directly synchronization CastleP30 using QuadSync, as the cameras point outwards around a courtyard, so that each camera only sees a small portion of the courtyard, and the higher-order viewing graph is very sparse. However, applying QuadSync on each of the handpick clusters then merging the results produces great reconstruction for the CastleP30 dataset. See Figure \ref{fig:CastleP30_handpicked_clusters} for a comparison of the retrieved poses and ground truth poses. 
\begin{figure}[htp!]
    \centering
    \includegraphics[width=0.6\linewidth]{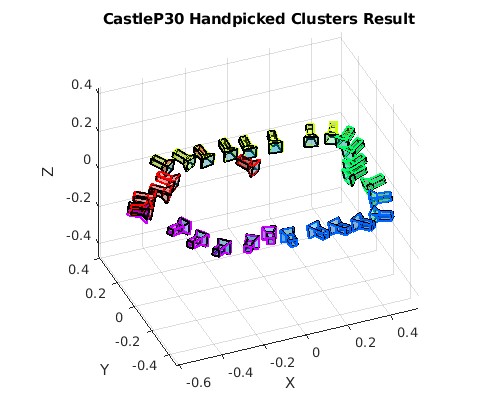}
    \caption{Retrieved poses (colored, where each cluster has a distinct color) vs. ground truth poses (black) for CastleP30 with handpicked clusters.}
    \label{fig:CastleP30_handpicked_clusters}
\end{figure}

\begin{table}[htpb!]
    \centering
    \begin{tabular}{ccccccc}
         \hline
         size & noise & $\bar{e_t}$ & $\hat{e_t}$ & $\bar{e_r}$ & $\hat{e_r}$ & time(s)\\
         \hline
         Full (30 cams) & 0 & 0 & 0 & 0 & 0 & 1666.3\\
         \hline
         Cluster1 (10 cams) & 0 & 0 & 0  & 0 & 0 & 27.3\\
         Cluster2 (15 cams) & 0 & 0 & 0 & 0 & 0 & 92.6 \\
         Cluster3 (15 cams) & 0 & 0 & 0 & 0 & 0 & 97.7 \\
         \hline
         Aligned (30 cams) & 0 & 0 & 0 & 0 & 0 & 150.2 \\
         \hline
         Full (30 cams) & 1 & 0.11 & 0.10 & 0.08 & 0.09 & 1944.8\\
         \hline
         Cluster1 (10 cams) & - & 0.66 & 0.63 & 0.21 & 0.29 & 31.9\\
         Cluster2 (15 cams) & - & 0.31 & 0.32 & 0.18 & 0.14 & 111.5 \\
         Cluster3 (15 cams) & - & 0.31 & 0.31 & 0.22 & 0.21 & 93.4 \\
         \hline
         Aligned (30 cams) & - & 0.53 & 0.44 & 0.35 & 0.32 & 247.5\\
         \hline
          Full (30 cams) & 2 & 0.28 & 0.28 & 0.18 & 0.19 & 1776.0\\
         \hline
         Cluster1 (10 cams) & - & 1.57 & 1.63 & 1.00 & 0.79 & 28.9\\
         Cluster2 (15 cams) & - & 0.77 & 0.76 & 0.41 & 0.36 & 112.2\\
         Cluster3 (15 cams) & - & 0.63 & 0.59 & 0.46 & 0.40 & 92.9 \\
         \hline
         Aligned (30 cams) & - & 1.14 & 0.97 & 0.74 & 0.62 & 245.7\\
         \hline
    
    \end{tabular}
    \caption{Results for synthetic distributed synchronization. $\bar{e_t}$: mean location error. $\hat{e_t}$: median location error. $\bar{e_r}$: mean rotation error. $\hat{e_r}$: median rotation error. Noise is added as percentage onto the cameras. In the noisy cases, only the quadrifocal tensors were synchronized.}
    \label{tab:distributed synchronization}
\end{table}

We note that we added uniform noise to the quadrifocal tensor estimates. Yet, the algorithm performs better in the corrupted setting, where some estimates are very accurate, and a smaller portion is very inaccurate. The main point that we would to demonstrate is that the algorithm can be significantly sped up when the viewing graph has obvious clusters without sacrificing too much accuracy. It is clear that the algorithm can be even further sped up when given more cores and computing resources. 

\section{Detailed Numerical Results}
In this section, we include more comprehensive results accompanying the results in the main paper. We report the completion rates and an idea of the runtime in Table \ref{tab:completion rate}. We also report the detailed mean location error, median location error, mean rotation error, median rotation error in Tables \ref{tab:mean_location_error},  \ref{tab:median_location_error}, \ref{tab:mean_rotation_error},
\ref{tab:median_rotation_error} respectively. Locations are reported in the metrics of their original datasets, and rotations are reported in degrees.

\begin{table*}[htbp]
  \centering
  \caption{Completion rates and runtimes of different methods}
  \label{tab:completion rate}
  \begin{tabular}{@{} l r r r @{}}
    \toprule
    Dataset & Completion \% & Joint Opt. (s) & TrifocalSync (s)\\
    \midrule
    courtyard (18/38) & 33.10 & 111.2 & 137.9\\
    electro (12/39) & 78.70 & 30.6 & 45.13\\
    facade (46/76) & 61.98 & 4226.2 & 1633.0\\
    kicker (20/30) & 31.84 & 157.3 & 174.3\\
    meadow (6/14) & 42.13 & 11.7 & 9.1\\
    office (13/21) & 39.92 & 34.2 & 53.3\\
    pipes (8/14) & 58.79 & 21.6 & 27.2\\
    relief\_2 (16/20) & \textbf{14.31} & 64.9 & 121.6\\
    relief (13/13) & 95.21 & 36.3 & 75.4\\
    terrace (11/23) & 92.53 & 27.5 & 54.8\\
    terrains (23/42) & \textbf{8.29} & 223.1 & 238.5\\
    CastleP19 (13/19) & \textbf{24.46} & 45.2 & 77.24\\
    CastleP30 (19/30) & \textbf{28.02} & 139.7 & 174.7\\
    EntryP10 (10/10) & \textbf{26.50} & 25.0 & 41.8\\
    FountainP11 (11/11) & 76.55 & 22.8 & 47.3\\
    HerzP8 (8/8) & 80.76 & 17.6 & 26.7\\
    HerzP25 (24/25) & \textbf{25.38} &  312.7 & 312.8\\
    \bottomrule
  \end{tabular}
\end{table*}
\begin{table*}[htbp]
  \centering
  \caption{Mean location error by method}
  \label{tab:mean_location_error}
  \begin{tabular}{@{} l r r r r r r r @{}}
    \toprule
    Dataset & QuadSync & \begin{tabular}[c]{@{}r@{}}Joint\\ Opt.\end{tabular} & \begin{tabular}[c]{@{}r@{}}Trifocal\\ Sync\end{tabular} & LUD & NRFM & \begin{tabular}[c]{@{}r@{}}MPLS\\ BATA\end{tabular} & \begin{tabular}[c]{@{}r@{}}MPLS\\ CS\end{tabular} \\
    \midrule
    courtyard (18/38) & 0.0477 & 0.0489 & 0.1753 & 0.0403 & 0.1104 & 0.0711 & \textbf{0.0234}\\
    electro (12/39) & 0.0199 & 0.0200 & \textbf{0.0197} & 0.0605 & 0.0257 & 0.0617 & 0.0604 \\
    facade (46/76) & 3.0294 & 0.0248 & 0.0222 & 0.2497 & 0.2068 & 0.0984 &\textbf{0.0158} \\
    kicker (20/30) & \textbf{0.0078} & \textbf{0.0078} & 0.0204 & 0.0366 & 0.0246 & 0.1934 & 0.0178 \\
    meadow (6/14) & \textbf{0.0401} & 0.1071 & 1.4341 & 0.6408 & 0.2207 & 1.8229 & 0.0691\\
    office (13/21) & 0.0039 & \textbf{0.0037} & 0.0151 & 0.0262 & 0.0187 & 0.0166 & 0.0086 \\
    pipes (8/14) & 0.0192 & 0.0115 & 0.0344 & 0.1097 & 0.1058 & 0.1379 & \textbf{0.0085}  \\
    relief\_2 (16/20) & 0.5210 & 0.5360 & 0.5554 & 0.8624 & 0.8526 & 0.2547 & \textbf{0.0158}\\
    relief (13/13) & \textbf{0.0020} & \textbf{0.0020} & \textbf{0.0020} & 0.0043 & 0.0047 & 0.1076 & 0.0052 \\
    terrace (11/23) & \textbf{0.0096} & \textbf{0.0096} & 0.0103 & 0.0165 & 0.0150 & 0.0163 & 0.0147 \\
    terrains (23/42) & 0.4631 & 0.4335 & 0.0797 & 0.0092 & 0.0149 & 0.0326 & \textbf{0.0075}\\
    CastleP19 (13/19) & 0.5130 & 0.4921 & 7.6932 & 0.2605 & 0.5965 & 1.3871 & \textbf{0.1094} \\
    CastleP30 (19/30) & 0.1629 & 0.1627 & 6.4738 & 0.1132 & 0.1683 & 0.4945 & \textbf{0.0687} \\
    EntryP10 (10/10) & \textbf{0.0007} & \textbf{0.0007} & 0.0654 & 0.0957 & 0.2913 & 0.1235 & 0.0897 \\
    FountainP11 (11/11) & \textbf{0.0002} & \textbf{0.0002} & 0.0098 & 0.0103 & 0.0223 & 0.2125 & 0.0095 \\
    HerzP25 (24/25) & \textbf{0.0025} & 0.0026 & 0.2187 & 0.0647 & 0.1114 & 9.2329 & 0.1668 \\
    HerzP8 (8/8) & \textbf{0.0010} & \textbf{0.0010} & 0.0180 & 0.0522 & 0.2299 & 0.5766 & 0.0375 \\
    \bottomrule
  \end{tabular}
\end{table*}

\begin{table*}[htbp]
  \centering
  \caption{Median location error by method}
  \label{tab:median_location_error}
  \begin{tabular}{@{} l r r r r r r r @{}}
    \toprule
    Dataset & QuadSync & \begin{tabular}[c]{@{}r@{}}Joint\\ Opt.\end{tabular} & \begin{tabular}[c]{@{}r@{}}Trifocal\\ Sync\end{tabular} & LUD & NRFM &\begin{tabular}[c]{@{}r@{}}MPLS\\ BATA\end{tabular}&\begin{tabular}[c]{@{}r@{}}MPLS\\ CS\end{tabular} \\
    \midrule
    courtyard & 0.0307 & 0.0324 & 0.0947 & 0.0308 & 0.0668 & 0.0598 & \textbf{0.0194} \\
    electro & 0.0151 & 0.0152 & 0.0133 & 0.0256 & \textbf{0.0081} & 0.0221 & 0.0181 \\
    facade & 2.5853 & 0.0135 & 0.0140 & \textbf{0.0108} & 0.0124 & 0.0253 & 0.0128 \\
    kicker & 0.0068 & \textbf{0.0063} & 0.0173 & 0.0090 & 0.0110 & 0.0788 & 0.0101 \\
    meadow & 0.0390 & 0.0922 & 1.1006 & 0.2374 & 0.1285 & 0.7635 & \textbf{0.0358} \\
    office & {0.0028} & \textbf{0.0027} & 0.0090 & 0.0035 & 0.0047 & 0.0043 & 0.0033\\
    pipes & 0.0139 & {0.0075} & 0.0212 & 0.0390 & 0.0506 & 0.1330 & \textbf{0.0046}  \\
    relief\_2 & 0.4392 & 0.4737 & 0.5204 & 0.7633 & 0.7852 & 0.2065 & \textbf{0.0072}\\
    relief & \textbf{0.0017} & \textbf{0.0017} & \textbf{0.0017} & 0.0048 & 0.0048 & 0.0052 & 0.0049 \\
    terrace & \textbf{0.0091} &\textbf{ 0.0091} & 0.0104 & 0.0173 & 0.0172 & 0.0164 & 0.0159 \\
    terrains & 0.4442 & 0.4052 & 0.0609 & \textbf{0.0065} & 0.0114 & 0.0165 & 0.0070 \\
    CastleP19 & 0.5502 & 0.4944 & 7.2040 & 0.1694 & 0.4583 & 1.1433 & \textbf{0.0913} \\
    CastleP30 & 0.1145 & 0.1169 & 5.6921 & 0.0698 & 0.1536 & 0.1670 & \textbf{0.0485} \\
    EntryP10 & \textbf{0.0006} & \textbf{0.0006} & 0.0514 & 0.0851 & 0.3107 & 0.1058 & 0.0571 \\
    FountainP11 & \textbf{0.0002} & \textbf{0.0002} & 0.0079 & 0.0086 & 0.0166 & 0.1889 & 0.0085 \\
    HerzP25 & \textbf{0.0009} & 0.0010 & 0.1876 & 0.0463 & 0.0849 & 0.1525 & 0.0608 \\
    HerzP8 & \textbf{0.0007} & \textbf{0.0007} & 0.0147 & 0.0545 & 0.0808 & 0.4254 & 0.0358 \\
    \bottomrule
  \end{tabular}
\end{table*}

\begin{table*}[htbp]
  \centering
  \caption{Mean rotation error by method}
  \label{tab:mean_rotation_error}
  \begin{tabular}{@{} l r r r r r r @{}}
    \toprule
    Dataset & QuadSync & \begin{tabular}[c]{@{}r@{}}Joint\\ Opt.\end{tabular} & \begin{tabular}[c]{@{}r@{}}Trifocal\\ Sync\end{tabular} & LUD & NRFM &\begin{tabular}[c]{@{}r@{}}MPLS\end{tabular} \\
    \midrule
    courtyard & 0.3059 & 0.3295 & 1.2398 & 0.1886 & 0.1886 & \textbf{0.1056} \\
    electro & 0.1022 & 0.1020 & 0.1208 & 0.0722 & 0.0722 & \textbf{0.0603} \\
    facade & 2.1185 & 0.1131 & 0.1239 & \textbf{0.0442} & \textbf{0.0442} & 0.0506 \\
    kicker & 0.2063 & 0.2028 & 0.2986 & \textbf{0.1162} & \textbf{0.1162} & 0.2697\\
    meadow & 0.2856 & 0.6100 & 12.4079 & \textbf{0.1403} & \textbf{0.1403} & 0.1861 \\
    office & 0.1653 & 0.1449 & 0.5749 & 0.0674 & 0.0674 & \textbf{0.0547} \\
    pipes & 0.5901 & 0.1785 & 3.0799 & 0.1516 & 0.1516 & 0.1546 \\
    relief\_2 & 32.6666 & 49.3035 & 52.5503 & 45.0175 & 45.0175 & \textbf{0.1049} \\
    relief & \textbf{0.0791} & \textbf{0.0791} & 0.0792 & 0.1080 & 0.1080 & 0.1077 \\
    terrace & 0.0796 & 0.0796 & 0.0812 & 0.0728 & 0.0728 & \textbf{0.0693} \\
    terrains & 11.5185 & 11.1041 & 3.7660 & 0.2027 & 0.2027 & \textbf{0.1918} \\
    CastleP19 & 30.4922 & 31.9495 & 39.5514 & 0.7604 & 0.7604 & \textbf{0.2795} \\
    CastleP30 & 15.8584 & 14.8579 & 10.6018 & 0.1937 & 0.1937 & \textbf{0.1168}  \\
    EntryP10 & \textbf{0.0822} & \textbf{0.0822} & 0.1489 & 0.2993 & 0.2993 & 0.3086 \\
    FountainP11 & 0.0734 & 0.0734 & 0.0912 & 0.0517 & 0.0517 & \textbf{0.0508} \\
    HerzP25 & \textbf{0.2271} & 0.2300 & 1.1558 & 0.2879 & 0.2879 & 0.7712  \\
    HerzP8 & \textbf{0.0997} & \textbf{0.0997} & 0.1167 & 0.4717 & 0.4717 & 0.3462 \\
    \bottomrule
  \end{tabular}
\end{table*}

\begin{table*}[htbp]
  \centering
  \caption{Median rotation error by method}
  \label{tab:median_rotation_error}
  \begin{tabular}{@{} l r r r r r r @{}}
    \toprule
    Dataset & QuadSync & \begin{tabular}[c]{@{}r@{}}Joint\\ Opt.\end{tabular} & \begin{tabular}[c]{@{}r@{}}Trifocal\\ Sync\end{tabular} & LUD & NRFM & \begin{tabular}[c]{@{}r@{}}MPLS\end{tabular} \\
    \midrule
    courtyard & 0.2144 & 0.2283 & 1.2651 & 0.1909 & 0.1909 & \textbf{0.1234} \\
    electro & 0.0883 & 0.0888 & 0.1186 & \textbf{0.0630} & \textbf{0.0630} & 0.0638\\
    facade & 0.9737 & 0.0731 & 0.0812 & \textbf{0.0449} & \textbf{0.0449} & 0.0457 \\
    kicker & 0.1735 & 0.1701 & 0.2877 & \textbf{0.1146} & \textbf{0.1146} & 0.1650 \\
    meadow & 0.2002 & 0.5049 & 7.3067 & \textbf{0.0902} & \textbf{0.0902} & 0.1337 \\
    office & 0.0994 & 0.1070 & 0.2807 & 0.0717 & 0.0717 & \textbf{0.0703} \\
    pipes & 0.2937 & \textbf{0.1131} & 1.7915 & 0.1310 & 0.1310 & 0.1312 \\
    relief\_2 & 7.9110 & 12.0705 & 19.4790 & 38.7239 & 38.7239 & \textbf{0.1039} \\
    relief & 0.0880 & 0.0880 & \textbf{0.0849} & 0.1173 & 0.1173 & 0.1157 \\
    terrace & \textbf{0.0860} & \textbf{0.0860} & 0.0885 & 0.0962 & 0.0962 & 0.0916 \\
    terrains & 2.6535 & 2.3353 & 2.6054 & \textbf{0.2280} & \textbf{0.2280} & 0.2567\\
    CastleP19 & 6.7486 & 8.3131 & 7.7256 & 0.5380 & 0.5380 & \textbf{0.2186}  \\
    CastleP30 & 7.9894 & 7.8164 & 4.2029 & 0.1520 & 0.1520 & \textbf{0.0903}  \\
    EntryP10 & \textbf{0.0810} & \textbf{0.0810} & 0.1173 & 0.2840 & 0.2840 & 0.2375  \\
    FountainP11 & 0.0724 & 0.0724 & 0.0884 & 0.0559 & 0.0559 & \textbf{0.0448} \\
    HerzP25 & \textbf{0.1500} & 0.1584 & 0.7872 & 0.2564 & 0.2564 & 0.4064 \\
    HerzP8 & \textbf{0.0813} & \textbf{0.0813} & 0.1107 & 0.4824 & 0.4824 & 0.3354 \\
    \bottomrule
  \end{tabular}
\end{table*}

\end{document}